\documentclass{article}

\usepackage{nnlinkage_preprint,times}
\usepackage{amsmath,amssymb,amsthm}
\usepackage{booktabs}
\usepackage{graphicx}
\usepackage[caption=false,font=normalsize]{subfig}
\usepackage{tikz}
\usepackage{url}
\usepackage{float}
\usepackage{wrapfig}
\usepackage{placeins}
\usetikzlibrary{arrows.meta,calc}

\floatstyle{ruled}
\newfloat{algorithm}{t}{loa}
\floatname{algorithm}{Algorithm}

\usepackage[hidelinks,unicode]{hyperref}
\hypersetup{
  pdftitle={Algorithmically Aligned Neural Agglomerative Tree Construction},
  pdfauthor={Robert R Nerem, Pranav Singh, Cheyenne Ward, Yusu Wang},
  bookmarksnumbered=true
}
\usepackage{cleveref}
\crefname{section}{Sec.}{Secs.}
\Crefname{section}{Sec.}{Secs.}
\crefname{subsection}{Sec.}{Secs.}
\Crefname{subsection}{Sec.}{Secs.}
\crefname{appendix}{App.}{Apps.}
\Crefname{appendix}{App.}{Apps.}
\crefname{equation}{Eq.}{Eqs.}
\Crefname{equation}{Eq.}{Eqs.}
\crefname{figure}{Fig.}{Figs.}
\Crefname{figure}{Fig.}{Figs.}
\crefname{subfigure}{Fig.}{Figs.}
\Crefname{subfigure}{Fig.}{Figs.}
\crefname{table}{Tab.}{Tabs.}
\Crefname{table}{Tab.}{Tabs.}
\crefname{algorithm}{Algorithm}{Algorithms}
\Crefname{algorithm}{Algorithm}{Algorithms}

\newtheorem{theorem}{Theorem}

\newtheorem{lemma}[theorem]{Lemma}
\newtheorem{proposition}[theorem]{Proposition}
\newtheorem{corollary}[theorem]{Corollary}

\crefname{corollary}{Cor.}{Cors.}
\Crefname{corollary}{Cor.}{Cors.}

\crefname{theorem}{Thm.}{Thms.}
\Crefname{theorem}{Thm.}{Thms.}
\crefname{definition}{Def.}{Defs.}
\Crefname{definition}{Def.}{Defs.}
\crefname{lemma}{Lem.}{Lems.}
\Crefname{lemma}{Lem.}{Lems.}
\crefname{proposition}{Prop.}{Props.}
\Crefname{proposition}{Prop.}{Props.}

\newcommand{\R}{\mathbb{R}}

\newcommand{\cC}{\mathcal{C}}
\newcommand{\cL}{\mathcal{L}}

\DeclareRobustCommand{\nnlinkage}{\textup{\textsc{NN-linkage}}}
\title{Algorithmically Aligned Neural \\Agglomerative Tree Construction}

\author{%
\textbf{Robert R Nerem}$^{1}$ \quad
\textbf{Pranav Singh}$^{1}$ \quad
\textbf{Cheyenne Ward}$^{2}$ \quad
\textbf{Yusu Wang}$^{1}$\\[0.5ex]
$^{1}$University of California San Diego\\
$^{2}$University of Washington
}

\begin{document}

\maketitle

\begin{abstract}

Linkage algorithms for hierarchical clustering (HC) are a powerful and efficient
framework for constructing clustering trees, yet it is often unclear which merge rule best suits a given dataset or task. In contrast, neural approaches can
learn from data, but often fail to retain
the efficiency and size generalization of classical algorithms. We introduce \nnlinkage{}, a neural network (NN)
model that can learn task-specific and locally dependent merge rules while retaining the recursive
structure and efficient inference of classical linkage algorithms. 
In particular, our model is algorithmically aligned with the Lance--Williams (LW) recurrence, a parameterized framework for defining a broad, continuous family of linkage rules for agglomerative HC. Classical methods such as single linkage (SL), complete linkage (CL), and average linkage arise as discrete choices within this broader family.
We show that \nnlinkage{} is a universal approximator for continuous linkage functions,
including LW recurrences, and, when paired with a transformer
encoding, can also approximate globally dependent rules such as robust single-linkage. We further
show that \nnlinkage{} can exactly implement any symmetric constant-coefficient LW recurrence across all input sizes.
On the empirical front, we evaluate \nnlinkage{} in real-world applications, clock-tree routing and phylogenetic reconstruction, using both synthetic and real datasets, demonstrating its effectiveness over both classical algorithms and other neural approaches.
By learning merge rules directly from target trees, \nnlinkage{} extends efficient
HC to scientific and engineering objectives not adequately
captured by existing hand-designed linkage rules.
\end{abstract}

\section{Introduction}

Traditional linkage algorithms for hierarchical clustering (HC) are efficient and
interpretable.
However, different linkage functions offer different benefits, so it is often unclear which rule is best
suited to a given dataset. Additionally, these merge rules remain fixed and are therefore
unadaptable to the data or task, especially if we need different local rules depending on different local point configurations. In contrast, neural methods excel at
using data to improve inference, but often struggle with size generalization
and may fail to leverage algorithmic structure for efficient computation. We aim
to develop a neural method for HC that retains the efficiency, recursive
structure, and size generalization of traditional linkage algorithms, while
using data to learn a flexible and (locally and globally) adaptable linkage function from an expressive class.

Another important advantage of supervised neural approaches is that they can address
both explicit and implicit objectives. For explicit objectives, such as
Dasgupta's HC cost, supervision can be provided by solving small instances
optimally, while size generalization can enable good solutions on larger
instances where exact optimization is impractical. Implicit objectives can also
be provided by labels, for example by hand-labeled data, or through phylogenetic trees generated under a
particular evolutionary model. In many engineering domains, such as very-large-scale integration (VLSI),
solutions are often constructed through lengthy pipelines combining
hand-selected heuristics and iterative refinement to optimize multiple, often
poorly specified objectives. By providing these solutions as supervision (labeled target trees), a
neural model may learn to compress this lengthy procedure while optimizing the
same implicit objective.

Neural algorithmic reasoning aims to combine the benefits of algorithms and neural networks (NNs)
by constructing NNs with an inductive bias
toward algorithmic behavior~\citep{VelickovicBlundell2021,He2026NeuralizationValue}. A common technique for introducing this bias is
algorithmic alignment, where NN architectures are designed to
reflect the computational structure of a target algorithm or algorithmic
framework~\citep{Xu2020AlgorithmicAlignment}.
Both empirical and theoretical studies have shown that algorithmic alignment
can enable NNs to generalize to larger input sizes than those seen
during training~\citep{Nerem2026ShortestPaths,Velickovic2022CLRS,Velickovic2020NeuralExecution}.

\paragraph{Our Contribution.} Following this approach, we introduce \nnlinkage{}, a neural model aligned with
the Lance--Williams (LW) recurrence framework for computing linkage functions. Our
model replaces scalar linkage scores with learned merge embeddings and the
LW update with a recurrently applied NN. This design allows
\nnlinkage{} to retain the recursive computation and priority-queue structure of LW
linkage algorithms while learning an expressive merge rule from data.
\nnlinkage{} offers greater flexibility than simply selecting among existing
linkage algorithms. In particular, the model can learn locally adaptive merge
strategies that depend on the clusters being merged or their context in the
input point set. These strategies come from an infinite family of linkage
functions and can be highly complex, making them difficult to specify through
explicit hand-designed rules.
Additionally, our model can maintain the $O(n^2\log n)$ time complexity of linkage functions computed by LW recurrences, by using a priority queue to rank potential merges. This priority queue is also
useful during training, as it allows us to compute a loss more efficiently without requiring enumeration of all pairwise comparisons between potential merges. 

On the theoretical front, we prove that \nnlinkage{} is universal in that, for a fixed input size, it can approximate any continuous linkage function of the two subtrees being merged. If the input is initialized with transformer embeddings, the model can approximate any
score that additionally depends on the entire point set. This global dependence
enabled by the transformer allows \nnlinkage{} to approximate linkage functions
such as robust single linkage (SL), which uses $k$-nearest-neighbor distances.
We also show that the \nnlinkage{} model can exactly realize any symmetric
constant-coefficient LW recurrence across all input sizes (instead of a fixed size).
These results show that \nnlinkage{} can replicate common linkage
functions and approximate general rules that depend on tree structure and
global context.

Empirically, across two tested applications, \nnlinkage{} outperforms classical algorithms and is competitive against other neural approaches. In the first domain, we perform empirical evaluations on
clock-tree routing, specifically the zero-skew tree problem in $L_1$, which can be reduced to HC under the sum of diameters objective. We train \nnlinkage{} with supervision coming from
uniformly sampled optimally solved 20-point instances. The resulting model generalizes to much larger synthetic instances, outperforming complete linkage (CL), the SoTA algorithm for this problem used in chip design, on point clouds with up to 500 points. Applied as a refinement of CL on VLSI benchmarks, it achieves Pareto-optimal performance and state-of-the-art objective values.

We also test \nnlinkage{}'s ability to recover phylogenetic trees from multiple
sequence alignment (MSA) data. Following the experimental setup of
\citet{Zhang2025NeuralNJ}, we use supervised training on instances generated through simulated
evolution~\citep{LyTrong2022AliSim} and evaluate it on both synthetic
and real data. On both datasets, \nnlinkage{} is Pareto optimal, running substantially faster than comparable neural
methods, with a modest reduction in tree accuracy on synthetic data and slightly improved accuracy on real data.

\paragraph{Related Work.}

HC is widely used to identify hierarchical structure in data and yield clusterings at
multiple levels of resolution. In seminal work, \citet{Dasgupta2016} introduced
a cost function on clustering trees based on pairwise similarities, allowing
HC algorithms to be evaluated through a common optimization objective.
\citet{CohenAddad2018} extend this framework to broader HC objectives and derive
approximation guarantees.
For the dual objective, \citet{MoseleyWang2017} showed that average linkage
gives a $1/3$-approximation. Statistical work studies convergence to population cluster trees, including merge distortion~\citep{Eldridge2015MergeDistortion}, axiomatic population HC~\citep{AriasCastroCoda2025AxiomaticHC}, and consistency of robust SL for density cluster trees~\citep{ChaudhuriDasgupta2010}.

Prior work has incorporated learning into HC through learned representations, objectives, and merge rules. JULE jointly learns representations and agglomerative cluster assignments~\citep{Yang2016JULE}, while HypHC optimizes a hyperbolic relaxation of Dasgupta's objective~\citep{Chami2020HypHC}. More closely related to our work, \citet{Balcan2017AlgorithmConfiguration,Balcan2019LearningToLink} study parameterized linkage families and metric mixtures with generalization guarantees, and Exponential Linkage learns a distance function and parameterized linkage rule~\citep{Yadav2019ExponentialLinkage}. These methods learn within prescribed linkage families, whereas \nnlinkage{} learns the linkage function itself through recurrent neural updates. NeuralNJ~\citep{Zhang2025NeuralNJ} also constructs an HC through learned merge rankings, but uses global attention over subtree representations and lacks alignment with the LW recurrence. We compared directly with NeuralNJ in the phylogenetic reconstruction experiment.

\nnlinkage{} also builds on algorithmic alignment in neural algorithmic reasoning~\citep{VelickovicBlundell2021,He2026NeuralizationValue,Xu2020AlgorithmicAlignment}. Alignment with dynamic programming is studied extensively by \citet{Dudzik2022DynamicProgrammers}, with theoretical work establishing conditions for size generalization~\citep{Wittig2026WhichAlgorithms,Nerem2026ShortestPaths}. Empirically, aligned models have learned standard algorithms~\citep{Ibarz2022Generalist}, improved combinatorial optimization through primal-dual structure~\citep{He2025PrimalDual}, and generalized across sizes in online matching~\citep{Hayderi2024MAGNOLIA}. NN-Steiner~\citep{Kahng2024NNSteiner} is another example of algorithmic alignment
with size generalization and, like NN-linkage, replaces computations within a
classical algorithm's recursive structure with neural networks.
\section{Model}

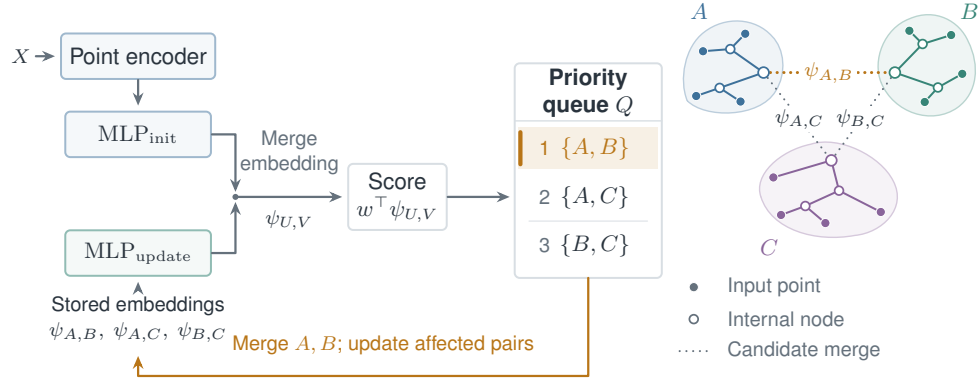
\begin{figure}[tbp]
\centering
\scalebox{0.94}{\definecolor{diagramink}{HTML}{29333D}
\definecolor{diagrammuted}{HTML}{65717C}
\definecolor{diagramrule}{HTML}{D8DEE3}
\definecolor{diagrama}{HTML}{3E719A}
\definecolor{diagramb}{HTML}{378577}
\definecolor{diagramc}{HTML}{89649A}
\definecolor{diagramselected}{HTML}{B87517}

\begin{tikzpicture}[
    x=1cm,y=1cm,
    font=\sffamily\fontsize{9}{11}\selectfont,
    text=diagramink,line cap=round,line join=round,
    flow/.style={-{Stealth[length=5.2pt,width=4.3pt]},draw=diagrammuted,line width=.9pt,shorten >=2pt},
    feedback/.style={flow,draw=diagramselected},
    block/.style={draw=diagramrule,fill=white,line width=.8pt,
      rounded corners=2pt,minimum height=.64cm,minimum width=2.05cm,align=center},
    tree/.style={line width=.8pt},
    point/.style={circle,draw=white,line width=.35pt,minimum size=3.7pt,inner sep=0pt},
    merge/.style={circle,fill=white,line width=.7pt,minimum size=3.5pt,inner sep=0pt},
    root/.style={merge,minimum size=4.3pt},
    candidate/.style={draw=diagrammuted,line width=.8pt,dash pattern=on 0pt off 2.6pt},
    chosen/.style={candidate,draw=diagramselected,line width=1.2pt},
    embedding/.style={fill=white,inner sep=1pt,font=\fontsize{8}{9}\selectfont},
    note/.style={font=\sffamily\fontsize{8}{10}\selectfont,text=diagrammuted},
    tensor/.style={font=\fontsize{8}{10}\selectfont,inner sep=2pt},
    heading/.style={font=\sffamily\bfseries\fontsize{9}{11}\selectfont}
]
\begin{scope}[xshift=9.60cm]

\begin{scope}[shift={(.05,1.94)},scale=.6]
\path[fill=diagrama!8,draw=diagrama!35,line width=.55pt]
  (.18,3.64) .. controls (-.02,4.13) and (.12,4.91) .. (.58,5.25)
  .. controls (1.07,5.63) and (1.78,5.59) .. (2.17,5.15)
  .. controls (2.46,4.86) and (2.52,4.28) .. (2.25,3.86)
  .. controls (1.87,3.23) and (.59,3.22) .. cycle;
\path[fill=diagramb!8,draw=diagramb!35,line width=.55pt]
  (4.73,4.06) .. controls (4.51,4.59) and (4.80,5.34) .. (5.39,5.56)
  .. controls (6,5.83) and (6.57,5.50) .. (6.77,5.03)
  .. controls (7.05,4.47) and (6.99,3.82) .. (6.52,3.48)
  .. controls (5.91,3.05) and (5.02,3.36) .. cycle;
\path[fill=diagramc!8,draw=diagramc!35,line width=.55pt]
  (1.91,1.38) .. controls (1.72,1.89) and (2.35,2.47) .. (3.08,2.62)
  .. controls (3.75,2.83) and (4.72,2.34) .. (5.01,1.76)
  .. controls (5.35,1.12) and (4.75,.54) .. (4.11,.48)
  .. controls (3.19,.31) and (2.15,.61) .. cycle;
\coordinate (Ar) at (1.97,4.30);
\coordinate (Ai) at (1.19,4.85);
\coordinate (Aj) at (.94,3.94);
\coordinate (A1) at (.55,4.80);
\coordinate (A2) at (1.54,5.21);
\coordinate (A3) at (.39,3.75);
\coordinate (A4) at (1.43,3.62);
\coordinate (Br) at (5.05,4.30);
\coordinate (Bi) at (5.70,4.97);
\coordinate (Bj) at (6.06,3.95);
\coordinate (B1) at (5.39,5.28);
\coordinate (B2) at (6.37,5.18);
\coordinate (B3) at (6.63,4.27);
\coordinate (B4) at (5.75,3.61);
\coordinate (Cr) at (3.56,2.24);
\coordinate (Ci) at (3.74,1.53);
\coordinate (Cj) at (2.95,1.17);
\coordinate (C1) at (2.20,1.84);
\coordinate (C2) at (2.38,.96);
\coordinate (C3) at (3.42,.77);
\coordinate (C4) at (4.72,1.03);
\draw[chosen] (Ar) -- node[embedding,text=diagramselected] {$\psi_{A,B}$} (Br);
\draw[candidate] (Ar) -- node[embedding,pos=.53] {$\psi_{A,C}$} (Cr);
\draw[candidate] (Br) -- node[embedding,pos=.53] {$\psi_{B,C}$} (Cr);
\draw[tree,diagrama] (Ar)--(Ai)--(A1) (Ai)--(A2) (Ar)--(Aj)--(A3) (Aj)--(A4);
\draw[tree,diagramb] (Br)--(Bi)--(B1) (Bi)--(B2) (Br)--(Bj)--(B3) (Bj)--(B4);
\draw[tree,diagramc] (Cr)--(C1) (Cr)--(Ci)--(C4) (Ci)--(Cj)--(C2) (Cj)--(C3);
\foreach \p in {A1,A2,A3,A4} \node[point,fill=diagrama] at (\p) {};
\foreach \p in {B1,B2,B3,B4} \node[point,fill=diagramb] at (\p) {};
\foreach \p in {C1,C2,C3,C4} \node[point,fill=diagramc] at (\p) {};
\foreach \p in {Ai,Aj} \node[merge,draw=diagrama] at (\p) {};
\foreach \p in {Bi,Bj} \node[merge,draw=diagramb] at (\p) {};
\foreach \p in {Ci,Cj} \node[merge,draw=diagramc] at (\p) {};
\node[root,draw=diagrama] at (Ar) {};
\node[root,draw=diagramb] at (Br) {};
\node[root,draw=diagramc] at (Cr) {};
\node[text=diagrama,anchor=south] at (.46,5.35) {$A$};
\node[text=diagramb,anchor=south] at (6.80,5.35) {$B$};
\node[text=diagramc,anchor=north] at (2.10,.58) {$C$};
\end{scope}

\node[point,fill=diagrammuted] at (.25,1.50) {};
\node[note,anchor=west] at (.60,1.50) {Input point};
\node[merge,draw=diagrammuted] at (.25,1.05) {};
\node[note,anchor=west] at (.60,1.05) {Internal node};
\draw[candidate] (.05,.60)--(.45,.60);
\node[note,anchor=west] at (.60,.60) {Candidate merge};
\end{scope}

\begin{scope}[xshift=-4.02cm]
\node[block,draw=diagrama!50,fill=diagrama!5] (encoder)
  at (6.05,4.75) {Point encoder};
\node[tensor] (input) at (4.40,4.75) {$X$};
\draw[flow] (input.east)--(encoder.west);
\node[block,draw=diagrama!50,fill=diagrama!5] (init) at (6.05,3.65) {$\mathrm{MLP}_{\mathrm{init}}$};
\draw[flow] (encoder.south)--(init.north);

\node[block,draw=diagramb!55,fill=diagramb!5] (update) at (6.05,1.98) {$\mathrm{MLP}_{\mathrm{update}}$};
\node[note,align=center,text=diagramink] (cache) at (6.05,1.05)
  {Stored embeddings\\$\psi_{A,B},\ \psi_{A,C},\ \psi_{B,C}$};
\draw[flow] (cache.north)--(update.south);

\coordinate (join) at (7.42,2.77);
\draw[flow] (init.east)--(7.42,3.65)--(join);
\draw[flow] (update.east)--(7.42,1.98)--(join);
\fill[diagrammuted] (join) circle (1.2pt);
\node[block,minimum width=1.35cm,minimum height=.88cm] (score)
  at (9.70,2.77) {Score\\$w^\top\psi_{U,V}$};
\draw[flow] (join)--(score.west);
\node[note,align=center,anchor=south] at (8.18,2.96) {Merge\\embedding};
\node[tensor,anchor=north] at (8.18,2.66) {$\psi_{U,V}$};

\path[fill=white,draw=diagramrule,line width=.8pt,rounded corners=2pt]
  (11.35,1.63) rectangle (13.42,4.65);
\node[heading,align=center] at (12.385,4.24) {Priority\\queue $Q$};
\draw[diagramrule,line width=.55pt] (11.35,3.84)--(13.42,3.84);
\path[fill=diagramselected!11] (11.40,3.16) rectangle (13.37,3.75);
\draw[diagramselected,line width=1.7pt] (11.43,3.23)--(11.43,3.68);
\node[note,text=diagramselected,anchor=west] at (11.57,3.46) {1};
\node[text=diagramselected] at (12.48,3.46) {$\{A,B\}$};
\node[note,anchor=west] at (11.57,2.77) {2};
\node at (12.48,2.77) {$\{A,C\}$};
\draw[diagramrule,line width=.4pt] (11.55,2.43)--(13.22,2.43);
\node[note,anchor=west] at (11.57,2.09) {3};
\node at (12.48,2.09) {$\{B,C\}$};
\draw[flow] (score.east)--(11.35,2.77);

\draw[feedback] (12.385,1.63)--(12.385,.25)--(6.05,.25)--(cache.south);
\node[note,text=diagramselected,fill=white,inner sep=2pt] at (9.50,.69)
  {Merge $A,B$; update affected pairs};
\end{scope}
\end{tikzpicture}}
\caption{\nnlinkage{} initializes merge embeddings from encoded points and recursively updates them from stored embeddings. A shared linear scorer orders candidate merges in a priority queue.}
\label{fig:neural_linkage_queue}
\end{figure}

\paragraph{Preliminaries.}

Let $X=(x_1,\ldots,x_n)$ be an input set. A hierarchical clustering (HC) of $X$ is a
full binary rooted tree $\mathcal T$ whose leaves are the elements of $X$. An agglomerative, or bottom-up, algorithm for producing $\mathcal T$ maintains
a partition $\cC_t$ of $X$ at each step $t$. The algorithm initializes
$\cC_0=\{\{x_i\}:i\in[n]\}$, chooses a pair of clusters
$q_t=\{A_t,B_t\}$ in $\cC_t$ at step $t$, and updates
$\cC_{t+1}=(\cC_t\setminus\{A_t,B_t\})\cup\{A_t\cup B_t\}$. Thus each step
replaces two clusters with their union.
This sequence of merges defines a tree. At each step $t$, each cluster
$A\in\cC_t$ is associated with a rooted binary subtree $\tau_A$ whose
leaves are exactly the elements of $A$. Initially, $\tau_{\{x_i\}}$ is the
single leaf $x_i$. When the algorithm merges $A_t$ and $B_t$, it creates a new
tree $\tau_{A_t\cup B_t}=(\tau_{A_t},\tau_{B_t})$, whose root is a new
parent node with unordered children $\tau_{A_t}$ and $\tau_{B_t}$, and assigns
this tree to the cluster $A_t\cup B_t$. After $n-1$ merges, we denote the final tree by $\mathcal T=\tau_X$; it spans the entire dataset $X$ and is the clustering tree output by the algorithm.

Many standard agglomerative HC algorithms are \emph{linkage} algorithms.
At each step, every pair of subtrees $\tau_A,\tau_B$
receives a scalar score $s(\tau_A,\tau_B)$, and the algorithm merges a
highest-scoring pair,
$\arg\max_{A,B\in\cC_t,\ A\ne B} s(\tau_A,\tau_B)$.
Standard examples include single linkage (SL), complete linkage (CL), average
linkage, and Ward linkage.
If the linkage functions admit constant-time computation of scores for pairs involving newly formed clusters, then a priority-queue implementation gives
 $O(n^2\log n)$-time inference.
The classical Lance--Williams (LW) recurrence gives a broad
template for such linkage functions~\citep{LanceWilliams1967}, where 
the score is updated as \footnote{\textcolor{black}{%
For such LW updates, the score $s(\tau_A,\tau_B;\mathcal H_t)$ also depends
on the chronological merge history $\mathcal H_t=(q_0,\ldots,q_{t-1})$, that is,
the order in which past merges occurred. However, for simplicity of
presentation, we leave this history dependence implicit in our notation.}}
\begin{equation}
\label{eq:lance_williams_prelim}
s(\tau_{A\cup B},\tau_C)={}
\alpha_A s(\tau_A,\tau_C)+\alpha_B s(\tau_B,\tau_C)
+\beta s(\tau_A,\tau_B)
+\gamma\left|s(\tau_A,\tau_C)-s(\tau_B,\tau_C)\right|,
\end{equation}
where the coefficients may depend on cluster sizes $|A|,|B|,|C|$.
Standard choices of LW parameters are listed in \cref{app:lw_parameters}.

\paragraph{\nnlinkage{}.}

We introduce \nnlinkage{}, a neural network (NN) model for HC
that is algorithmically aligned with LW-style linkage approaches. In particular, 
\nnlinkage{} replaces the scalar quantities maintained by the LW recurrence
with learned merge embeddings. For each pair of current clusters $A,B$, it
maintains a cached embedding $\psi_{A,B}\in\R^d$. The LW update in
\cref{eq:lance_williams_prelim} is replaced with a shared neural update acting
on $\psi_{A,B}$, $\psi_{A,C}$, and $\psi_{B,C}$.

\nnlinkage{} first computes leaf embeddings with a permutation-equivariant encoder,
$h_i^X=E_\theta(X)_i$, where we write $h_i$ when $X$ is clear. We initialize the
embeddings for merges between singletons using a multilayer perceptron (MLP):
\begin{equation}
\label{eq:init_pair}
\psi_{\{x_i\},\{x_j\}}
=
\mathrm{MLP}_{\mathrm{init}}\!\left([h_i+h_j,\ |h_i-h_j|]\right).
\end{equation}
In the simplest case, leaf embeddings may be computed by a pointwise MLP $E_\theta(X)_i=\mathrm{MLP}_{\mathrm{leaf}}(x_i)$.
Alternatively, $E_\theta$ may be a transformer, allowing embeddings to depend on the entire point set $X$.

If \nnlinkage{} merges clusters $A$ and $B$, the merge
embedding for pairs $\{A\cup B,C\}$ is updated as
\begin{equation}
\label{eq:update}
\color{black}
\psi_{A\cup B,C}
=
\mathrm{MLP}_{\mathrm{update}}\!\left(
[\psi_{A,C}+\psi_{B,C},\
|\psi_{A,C}-\psi_{B,C}|,\
\psi_{A,B}]
\right).
\end{equation}

\begin{wrapalgorithm}{R}{0.50\textwidth}
\caption{\nnlinkage{}}
\label{alg:neural_linkage}
\normalsize
\begin{tabbing}
\quad\=\quad\=\kill
\textbf{Input} $X=(x_1,\ldots,x_n)$\\
$h_i\gets E_\theta(X)_i$\quad $(i\in[n])$\\
$\tau_{\{x_i\}}\gets x_i$\quad $(i\in[n])$\\
$\cC\gets\{\{x_i\}:i\in[n]\}$,\quad $Q\gets\emptyset$\\
\textbf{for} $1\le i<j\le n$ \textbf{do}\\
\> $p\gets\{x_i\}$,\quad $q\gets\{x_j\}$\\
\> $\psi_{p,q}\gets\mathrm{MLP}_{\mathrm{init}}([h_i+h_j,|h_i-h_j|])$\\
\> $Q.\operatorname{push}(\{p,q\},w^\top\psi_{p,q})$\\
\textbf{end for}\\
\textbf{while} $|\cC|>1$ \textbf{do}\\
\> \textbf{repeat}\\
\>\> $\{A,B\}\gets Q.\operatorname{popmax}()$\\
\> \textbf{until} $A,B\in\cC$\\
\> $U\gets A\cup B$,\quad $\tau_U\gets(\tau_A,\tau_B)$\\
\> $\cC\gets(\cC\setminus\{A,B\})\cup\{U\}$\\
\> \textbf{for} $C\in\cC\setminus\{U\}$ \textbf{do}\\
\>\> Compute $\psi_{U,C}$ using \cref{eq:update}.\\
\>\> $Q.\operatorname{push}(\{U,C\},w^\top\psi_{U,C})$\\
\> \textbf{end for}\\
\textbf{end while}\\
\textbf{return} $\mathcal T =\tau_X $
\end{tabbing}
\end{wrapalgorithm}

The map $\mathrm{MLP}_{\mathrm{init}}$ is the initialization MLP and is called
once for every singleton pair. The map $\mathrm{MLP}_{\mathrm{update}}$ is the
recursive update MLP and is called after each merge to compute the new merge
embeddings involving the newly formed cluster. The sum and absolute-difference features make the update MLP invariant to exchanging
the two merged clusters. 

The merge score is then computed as
$s_\theta^X(\tau_A,\tau_B)=w^\top\psi_{A,B}$,
where a higher score means higher merge priority, and the superscript $X$ makes
explicit the global dependence available through the leaf encoder.
We maintain the set of
potential merges in a priority queue ordered by $s_\theta^X$. At each step, the
algorithm removes the highest-scoring candidate whose clusters are still in the
current partition and merges it. This merge creates a new cluster $A\cup B$.
We then compute the merge embedding
$\psi_{A\cup B,C}$ for all other clusters $C$ and insert $\{A\cup B,C\}$ into the
priority queue with score \begin{equation}
    s_\theta^X(\tau_{A\cup B},\tau_C)=w^\top\psi_{A\cup B,C}.
\end{equation}
The algorithm repeats this process until a single tree remains, which occurs
after $n-1$ merges.
Candidates
whose clusters are no longer in the current partition are left in the queue and
discarded when they are popped. \cref{alg:neural_linkage} summarizes
the procedure.

We define $\mathcal{T}_\theta(X)$ to be the tree constructed by
\nnlinkage{} on input point set $X$. We now give two basic
properties our model satisfies, with proofs given in \cref{app:proofs}.

\begin{proposition}[Permutation equivariance]
\label{thm:equivariance}
\nnlinkage{} is permutation equivariant; with equivariant tie-breaking, every leaf
permutation $\pi$ satisfies $\mathcal{T}_\theta(\pi X)=\pi \mathcal{T}_\theta(X)$.
\end{proposition}

\begin{proposition}[Time complexity]
\label{lem:running_time}
For fixed network size and embedding dimension, \nnlinkage{} runs in
$O(n^2\log n)$ time if $E_\theta(X)$ can be computed in $O(n^2)$ time.
\end{proposition}

\paragraph{Training Objective.}

Now that we have described \nnlinkage{} and its inference path, we give our training algorithm, which uses pairwise comparisons between candidate merges. Fix a target tree $\mathcal{T}^\star$, and write $\mathcal C(\mathcal{T}^\star)$ for the set of
clusters appearing as nodes of $\mathcal{T}^\star$. At a training step $t$, we assume the
current partition contains only clusters in $\mathcal C(\mathcal{T}^\star)$. We then take
as positive examples exactly the current pairs whose union is another cluster
of the target tree and make every other pair a negative example. Let $P_t$ and
$N_t$ be the positive and negative pairs, respectively, with 
\begin{equation}
\label{eq:positive_negative_sets}
\begin{aligned}
P_t &= \{\{A,B\}: A,B\in\cC_t,\ A\ne B,\ A\cup B\in \mathcal C(\mathcal{T}^\star)\},\\
N_t &= \{\{A,B\}: A,B\in\cC_t,\ A\ne B\}\setminus P_t.
\end{aligned}
\end{equation}
As there are many merge orders that can generate a given tree, treating all
compatible merges as positive avoids imposing an arbitrary ordering on these
merge operations.
\textcolor{black}{For a candidate pair $q=\{A,B\}$ in the current partition,
we abbreviate $s_\theta^X(\tau_A,\tau_B)$ as $s_\theta(q)$.}
For any margin $m>0$, our loss penalizes each pair in which a negative is not
scored at least $m$ below a positive. In particular,
\begin{equation}
\label{eq:hinge_loss}
\cL_t
=
\sum_{p\in P_t}\sum_{q\in N_t}
\bigl[m+s_\theta(q)-s_\theta(p)\bigr]_+.
\end{equation}
The loss supervises the relative ordering of candidate merges without enforcing
any specific value for the scores.
The full loss for a training trajectory is $\cL=\sum_{t=0}^{n-2}\cL_t$.
We take the hinge subgradient to be zero at zero.

\begin{proposition}[Local ranking margins imply tree recovery]
\label{thm:ranking_margin_recovery}
Fix a target tree $\mathcal{T}^\star$. Suppose that at every state reached by merges
that form clusters in $\mathcal C(\mathcal{T}^\star)$, every pair $p\in P_t$, $q\in N_t$ satisfies
$s_\theta(p)\ge s_\theta(q)+m$ for some $m>0$.
Then \cref{alg:neural_linkage} returns $\mathcal{T}^\star$.
\end{proposition}
The proof is given in \cref{app:proofs}. The proposition considers all states reached by valid merge sequences. In
practice, we do not enumerate this family and instead train along a single trajectory.
For unordered target-tree supervision, the merge chosen during training is
the highest-scoring merge whose union is a cluster in $\mathcal{T}^\star$.

At a state with $k=|\cC_t|$ clusters, the positive pairs are disjoint while negative pairs may overlap; thus $|P_t|\le k/2=O(n)$, and
$|N_t|=O(n^2)$. The loss at this step may therefore contain $O(n^3)$ positive-negative comparisons. The strength of the
hinge-based loss is that, using the priority queue structure, we can ignore all
negatives that already have a margin greater than $m$ with the lowest-scoring
positive example. We can then scan only the high-scoring negatives and restrict
loss and gradient computation to comparisons involving these negatives.

The negatives that can contribute nonzero hinge loss (and its gradient) are
\begin{equation}
  H_t
=
\left\{q\in N_t:
s_\theta(q)>\min_{p\in P_t}s_\theta(p)-m
\right\}.  
\end{equation}
For any $q\notin H_t$ and any $p\in P_t$, we have
$m+s_\theta(q)-s_\theta(p)\le 0$. The corresponding hinge term is therefore
zero, and
also contributes zero gradient. Thus the full loss in \cref{eq:hinge_loss}
has the same value and gradient if $N_t$ is replaced by $H_t$. Because the priority queue orders candidates by score, the
implementation can scan negatives in descending order, stop once it reaches the
threshold defining $H_t$ above, and compute the exact gradient
using only these positive-negative pairs. This pruning does not improve the
cubic worst-case comparison count per step. However, if only a constant number
$K$ of the highest-scoring violating negatives are retained at each step, then
the total number of positive-negative comparisons across all $n-1$ steps is
$O(n^2)$.

\paragraph{Expressivity.}

Now that we have defined our model, we analyze its expressivity. First, we show
that it is universal in its ability to replicate linkage functions. As with many
universality results, this is a fixed-size statement, meaning the number of
parameters needed for approximation may grow with the number of input points. 
On the other hand, we show that we can establish an expressivity guarantee for unbounded input sizes using a model with constant complexity if we restrict ourselves to any symmetric constant-coefficient LW recurrence. Proofs of the expressivity results are given in \cref{app:proofs}.
\begin{theorem}[Fixed-size linkage universality]
\label{thm:contextual_linkage_universality}
Fix an input size $N$ and a compact domain
$\mathcal K\subseteq\R^p$, with $x_i\in\mathcal K$ for every input point.
With pointwise MLP leaf
embeddings $h_i$, the score \textcolor{black}{$s_\theta^X(\tau_A,\tau_B)$} can uniformly approximate any continuous linkage score $s(\tau_A,\tau_B)$ that depends on the
two candidate subtrees and is invariant to tree isomorphism and swapping $A$
and $B$. If $E_\theta$ is a transformer,
the same holds for continuous scores $s^X(\tau_A,\tau_B)$ that equivariantly depend on
the full input $X$. 
\end{theorem}
The primary challenge in proving this result is that \nnlinkage{} does not receive candidate subtrees directly, but instead observes them implicitly through the recurrently computed merge embeddings. We proceed by showing $\mathrm{MLP}_{\mathrm{update}}$  can approximate an injective DeepSets-style encoding of the two child subtrees at every merge. Induction over the merge steps then establishes that these embeddings suffice to approximate any continuous function of the candidate subtrees.
\begin{theorem}[Size-invariant LW]
\label{thm:lw}
Given exact singleton scores, every constant-coefficient LW recurrence with $\alpha_A=\alpha_B$
can be realized \textbf{across all input sizes} exactly by \nnlinkage{}.
\end{theorem}
This result follows from \nnlinkage{}’s algorithmic alignment with LW updates.  Each instance of $\mathrm{MLP}_{\mathrm{update}}$  accesses precisely the three merge embeddings required by the recurrence, and can compute any linear function of these merges exactly. This alignment allows a fixed-size local update to reproduce the classical algorithm exactly, regardless of input size, rather than simply approximate its behavior at a fixed size.

 Dependence on the candidate trees is an important feature of our model. While most linkage rules define the score of a candidate merge only as a function of the
two clusters $A$ and $B$, some common linkage functions allow the score to depend on the current subtrees
$\tau_A$ and $\tau_B$. For example, the weighted pair-group method with
arithmetic mean (WPGMA) updates scores by averaging the two child
scores (no weighting by cluster size), so two clusters with the same
leaf set can yield different scores if they were produced by different
merge histories. Median linkage, which can be represented with constant coefficients, is similar as the cluster representative is the
midpoint of its two child representatives, so it can also depend on the order in
which the leaves were merged. 

Additionally, robust single linkage (robust SL) \citep{ChaudhuriDasgupta2010} demonstrates the power of scores $s^X(\tau_A,\tau_B)$ that depend on the entire input $X$, rather than only the candidate subtrees. Its density-aware construction yields a statistically consistent estimator of the underlying density cluster tree \citep{ChaudhuriDasgupta2010}. In robust SL, each leaf pair is assigned an edge weight $\lambda_{ij}=\max\{r_k(x_i),r_k(x_j),|x_i-x_j|\}$, where $r_k(x_i)$ is the distance from $x_i$ to its $k$th nearest neighbor. Standard SL is then applied to these edge weights. A point in a sparse region has a larger nearest-neighbor radius, so it will not be merged until late, reducing the effect of outliers. The $k$-nearest-neighbor ($k$-NN) radius depends on $X$, so a pointwise MLP encoder cannot determine $r_k(x_i)$ from $x_i$ alone. Transformer leaf embeddings, however, make this global information available before the agglomerative updates begin, allowing \nnlinkage{} to express robust SL.

\begin{corollary}[Robust SL]
\label{cor:robust-sl}
For fixed $N$ and $1\le k<N$, transformer-based \nnlinkage{} can uniformly approximate robust SL linkage scores on compact input domains.
\end{corollary}

\section{Experiments}

We perform experiments in two different real-world domains to demonstrate the effectiveness and versatility of our model. In both experiments, we use a hinge margin $m=0.5$. Each model is trained on a single NVIDIA RTX A6000 GPU. 
We first begin with a simple synthetic dataset to show the utility of learned rules over explicit algorithms. 

\paragraph{Rings and Disks.}

We evaluate on a synthetic dataset inspired by the rings-and-disks dataset of \citet{Balcan2019LearningToLink}, where the goal is to recover four clusters corresponding to two concentric rings and two disks. This construction combines two geometries that favor different linkage rules. 
In the ring motif, SL is favored as it connects nearby points along each ring before connecting points between rings. In the disk motif, CL is favored as it merges compact groups within each disk before connecting the two disks.

\citet{Balcan2019LearningToLink} show that a learned convex combination of SL and CL can balance these objectives.
We construct a more challenging variant of this dataset by randomizing the placement of the ring and disk pairs and reducing the separation within each pair. 
On this dataset, no choice of a fixed mixing parameter $\alpha$ in the convex combination of SL and CL achieves perfect recovery of all four clusters. Evaluation uses the best four-subtree pruning of the inferred HC tree.
The challenge is therefore not simply to identify a good linkage rule for the
dataset, but to choose different merge behavior in different regions of the
same instance.

We train \nnlinkage{} by using the desired clustering tree for each rings-and-disks instance as supervision.
On 100 held-out 120-point instances with independently generated motif placements,
\nnlinkage{} correctly classifies 99.60\% of points.
\begin{wrapfigure}{r}{0.5\textwidth}
\centering
\includegraphics[width=\linewidth]{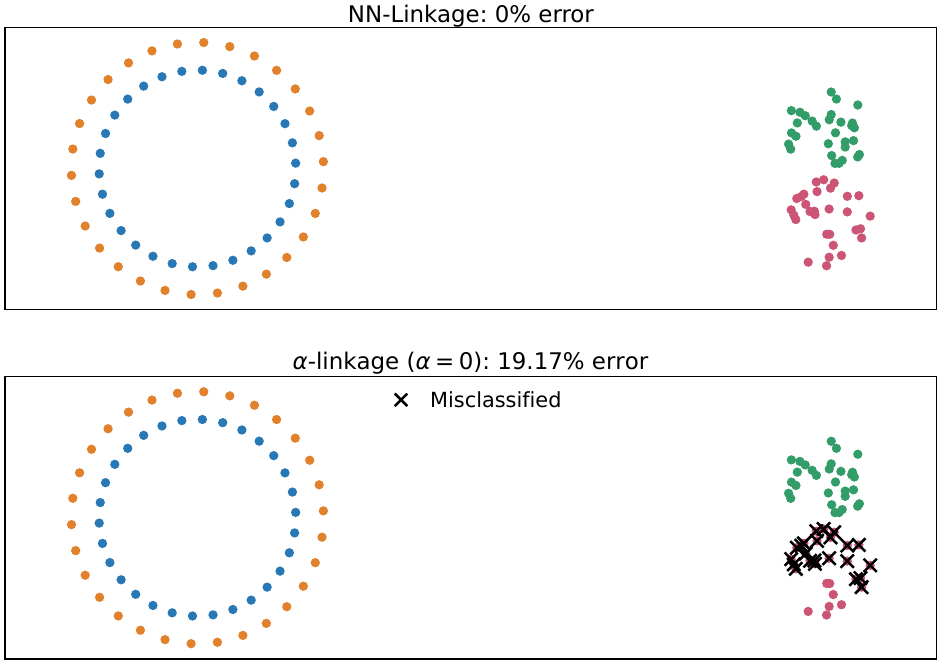}
\caption{A single Rings-and-disks test instance.  Colors indicate ground-truth clusters; crosses mark misclassified points.}
\label{fig:rings-disks}
\end{wrapfigure}
In contrast, the $\alpha$-linkage method of \citet{Balcan2019LearningToLink}
achieves 85.18\% accuracy.
We find that, indeed, \nnlinkage{} appears to have learned to automatically apply an SL-like strategy when the local points form the ring motif (for the ring motif, on average our model chooses a merge with minimum intercluster distance almost identical to the choice made by SL merges), while the model applies a CL-like strategy when locally observing the disk motif. Note these merge choices occur within the same HC tree, but with different rules applied depending on the local point configurations. Such rules in general would have been challenging to handcraft and encode in an explicit algorithm.
An example instance is given in \cref{fig:rings-disks}. These results show that \nnlinkage{} goes beyond applying a fixed linear combination of classical merge rules, and can instead employ a learned, {\bf locally dependent} merge rule. Full results are given in \cref{app:rings-disks}.

\paragraph{Clock Tree Routing.}

We evaluate \nnlinkage{} on zero-skew clock-tree routing, which seeks a minimum-length rooted Steiner tree in the $L_1$ plane with equal-length root-to-leaf paths \citep{Chao1992DME}. For many years, the standard practical method for this problem was Greedy-DME \citep{edahiro1993clustering,edahiro1994efficient}, itself an agglomerative algorithm. More recently, \citet{ChenYoung2019DimSum} showed a reduction to hierarchical clustering under the $L_1$ sum-of-diameters cost $
\sum_{A\in\mathcal C(\mathcal T)}\operatorname{diam}(A)
$, 
where \(\mathcal C(\mathcal T)\) contains the leaf sets of all nodes in the full clustering tree \(\mathcal T\) and
$\operatorname{diam}(A)=\max_{x,y\in A}\|x-y\|_1$.
They further show CL is a constant-factor approximation algorithm for this problem, and demonstrate state-of-the-art performance with CL,  outperforming greedy-DME. Thus, despite its simplicity, CL is a strong baseline following a long-standing line of work. We show \nnlinkage{} can learn a better agglomerative merge rule for optimizing this sum-of-diameters objective.

For training, each instance contains 20 points uniformly sampled from $[0,1]^2$,
and their optimal clustering tree computed using the exact dynamic program from
DimSum~\citep{ChenYoung2019DimSum}. We train \nnlinkage{} on 500,000 such point sets. We use \cref{eq:hinge_loss} restricted to the highest-scoring positive merge, and consider only the top 50 margin comparisons, thereby reducing time to compute the loss. We also provide additional features to $\mathrm{MLP}_{\mathrm{update}}$, details of which are given in \cref{app:experiment_details}.

\begin{figure}[t]
\centering
\subfloat[Uniform-point size generalization.\label{fig:uniform-sizegen}]{%
\includegraphics[width=0.49\linewidth,trim={1.970bp 2.006bp 1.435bp 1.435bp},clip]{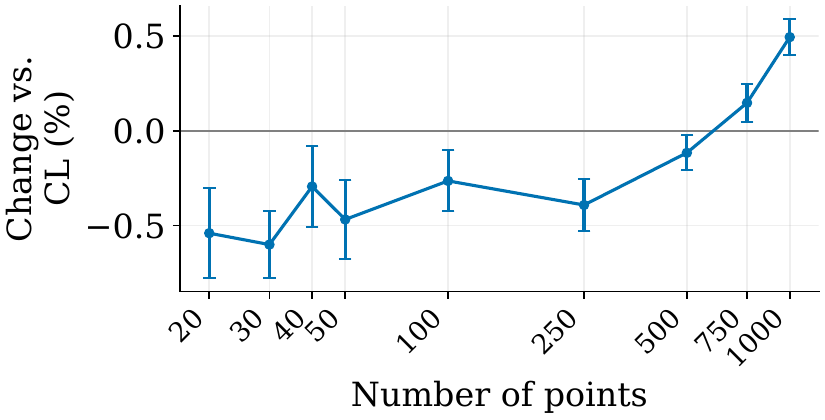}}\hfill
\subfloat[VLSI benchmarks.\label{fig:vlsi-pareto}]{%
\includegraphics[width=0.49\linewidth,trim={1.970bp 2.834bp 1.435bp 1.435bp},clip]{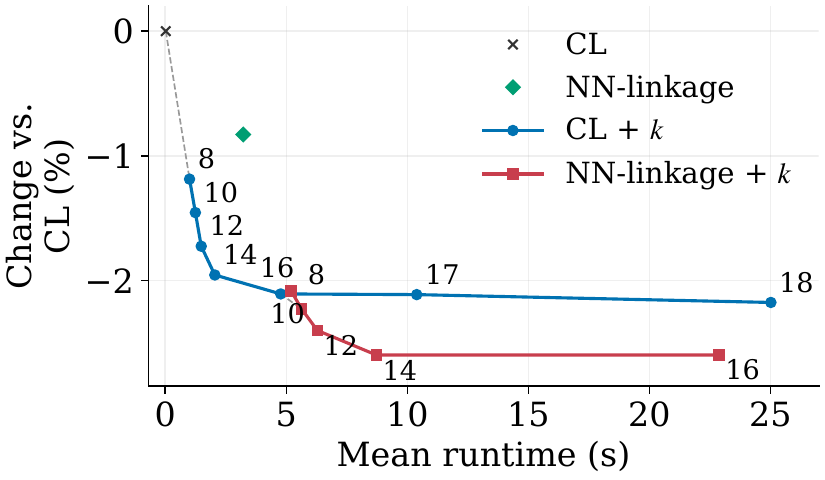}}
\caption{Clock-tree results. Mean percentage change in the sum-of-diameters
objective relative to raw CL; negative values indicate improvement.
(a) Means over 50 held-out point sets per size, with error bars showing one
standard error. (b) Mean objective change and runtime over 20 VLSI instances;
point labels indicate the refinement parameter $k$.}
\label{fig:clock-tree-results}
\end{figure}

We evaluate this \nnlinkage{} model on uniformly sampled point sets with up to $n=1000$ points, despite training only on $n=20$. As shown in \cref{fig:uniform-sizegen}, \nnlinkage{} improves over CL on average through $n=500$, demonstrating substantial size generalization. Detailed results are given in \cref{app:experiment_details}.
In \cref{app:complete-linkage-counterexample}, we give a four-point instance where CL incurs $5/4$ of the optimal cost, while trained \nnlinkage{} recovers the optimum. This gives a concrete example of how \nnlinkage{} has learned to outperform traditional algorithms.

We next test our trained model on 20 instances from the BST-DME and ISPD benchmark collections, listed in \cref{tab:vlsi-instances}. Following~\citet{ChenYoung2019DimSum}, we compare against CL and CL with refinement. The refinement algorithm repeatedly selects $k$ disjoint subtrees and uses the exact dynamic
program to compute the optimal clustering over these $k$ trees. Increasing $k$ allows larger changes but makes
each refinement exponentially more expensive.
As our model has variable performance on these benchmark instances due to distribution shift in size and point set structure, we apply our model as a post-processing optimization to CL.
Specifically, we start from the raw CL tree and recursively propose \nnlinkage{}-generated
replacements for its subtrees, visiting the root and up to six child levels.
As our model offers better performance for point sets of size less than 750, we run \nnlinkage{} on subtrees of size 16 to 750, and replace the CL-generated subtree with the model-generated subtree when it is superior.
This improves the objective by 0.829\% on average relative to raw CL with a mean runtime
of 3.219\,s, compared with 0.018\,s for raw CL. 

\label{sec:vlsi-pareto}
Finally, \cref{fig:vlsi-pareto} compares CL
plus local $k$-refinement against the recursive model with local
$k$-refinement. 
Refined model runs start from CL+$k$, apply the same $k$-refinement to each
model replacement candidate, and accept only objective improvements. \nnlinkage{} extends the Pareto frontier, particularly for higher refinement levels.
For example, compared to CL with $k=17$, the model with $k=14$ achieves an additional
0.485-percentage-point mean objective reduction relative to raw CL, while
taking 16\% less time.
\cref{tab:vlsi-pareto} in
\cref{app:experiment_details} gives detailed results.

\paragraph{Phylogenetic Reconstruction.}
\label{sec:phylogenetics}
\begin{wrapfigure}{R}{0.5\textwidth}
\centering
\includegraphics[width=\linewidth,trim={19.754bp 7.100bp 12.258bp 3.230bp},clip]{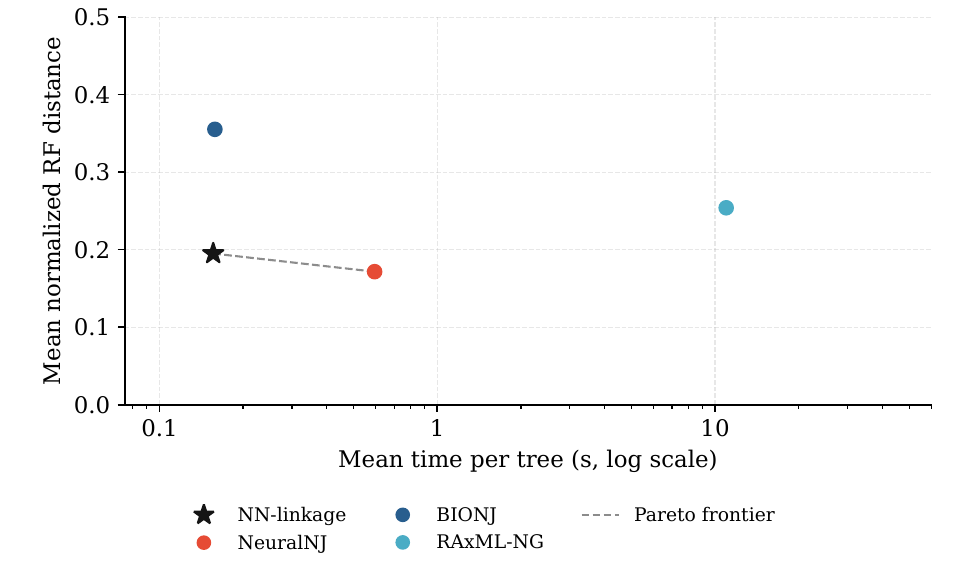}
\caption{Synthetic data comparison. Mean normalized RF distance and runtime
over 1,152 instances, equally weighting the nine size--length settings.
Dashed lines join nondominated points. Lower RF and runtime are better.}
\label{fig:phylo-synthetic-pareto}
\end{wrapfigure}

We also test our model on phylogenetic tree reconstruction, showing improved performance over classical and neural approaches. In this problem the aim is to infer evolutionary relationships from molecular
sequences. Each instance is a multiple sequence alignment (MSA), whose
constituent sequences correspond to the leaves of the HC tree.
  We use the experimental setting and baselines of \citet{Zhang2025NeuralNJ} and compare with their model NeuralNJ, which was shown to achieve state-of-the-art performance, including outperforming popular models like Phyloformer~\citep{Nesterenko2025Phyloformer}. We evaluate tree topology using normalized
Robinson--Foulds (RF) distance, the fraction of splits differing from the target tree.

 Tree-labeled training data is generated by simulated evolution~\citep{LyTrong2022AliSim} with between 20 and 100 leaves. Reaching the selected checkpoint took
\textbf{2.8 days} including validation, which, compared to NeuralNJ's reported \textbf{15 days} on an
A800 GPU~\citep[Supplementary Tab.~12]{Zhang2025NeuralNJ}, offers
a substantial practical improvement.\footnote{%
The RTX A6000 and A800 have boost clocks of 1.80 and
1.41\,GHz, respectively. The reported training-time difference is substantially
larger than this clock-frequency ratio, although differing GPU architectures
prevent a controlled comparison.}
We evaluate \nnlinkage{} against the released greedy NeuralNJ model and two classical algorithms, BIONJ and RAxML-NG. The synthetic test set contains all combinations of leaf count (20, 50, 100) and sequence length (256, 512, 1024), and is sampled using the same procedure as the training data. \cref{fig:phylo-synthetic-3x3-pareto} reports results for individual settings.

We next evaluate the same checkpoint on four datasets from
\citet{Zhou2018PhylogenomicBenchmarks}, which contain real molecular sequence
instances and their best-observed reference trees (\cref{fig:phylo-real-pareto}). Because these references
were themselves inferred with classical methods, agreement with them can
artificially inflate the apparent accuracy of likelihood-based reconstruction.
We therefore compare only against the neural baseline NeuralNJ here, demonstrating the ability of \nnlinkage{} to generalize out of distribution.  Across these instances, \nnlinkage{} achieves a mean RF
distance of .6075, compared with .6221 for NeuralNJ. Mean runtime is
.339\,s per instance versus 1.701\,s, approximately a 5$\times$ speedup.

\begin{figure}
\centering
\includegraphics[width=\linewidth,trim={6.182bp 7.802bp 5.634bp 6.856bp},clip]{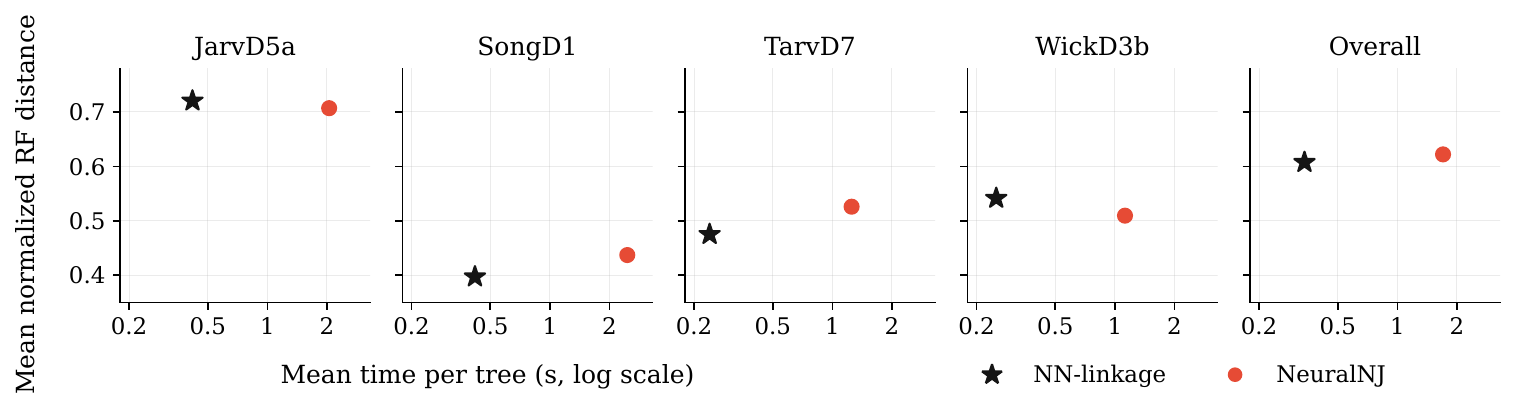}
\caption{Real phylogenetic data comparison. Mean normalized RF distance and
runtime on all 26,136 instances successfully processed by both models. Overall pools individual instances, rather than
equally weighting datasets. Lower RF and runtime are better.}
\label{fig:phylo-real-pareto}
\end{figure}

\nnlinkage{} is significantly faster than NeuralNJ on all benchmarks. It trades a modest increase in mean RF distance for speed on the
synthetic instances, while achieving a slightly improved RF distance on real phylogenetic data. The reversal in performance between in-distribution synthetic and real data suggests that \nnlinkage{} offers improved out-of-distribution generalization over NeuralNJ. Compared with BIONJ, which belongs to the class of high-speed distance-based methods, our model achieves comparable runtime with significantly better tree recovery. Compared with RAxML-NG, a leading maximum-likelihood approach, \nnlinkage{} improves in both tree recovery and runtime.

\section{Conclusion}

NNs are highly effective across many tasks, yet classical algorithms remain
central to applications such as VLSI design. \nnlinkage{} combines the recursive
structure of linkage algorithms with learned merge rules. It retains efficient
inference while achieving favorable quality--runtime tradeoffs in clock-tree
routing and phylogenetic reconstruction.

Future work includes extending this approach to objectives such as Dasgupta's
HC cost, which is NP-hard to minimize~\citep{Dasgupta2016}. Average linkage
provides a constant-factor approximation for its dual
objective~\citep{MoseleyWang2017}, raising the question of whether learned
agglomerative rules can improve this approximation guarantee. The relationship
between training data and algorithmic inductive bias is another important
direction, as suggested by results for Bellman--Ford-aligned NNs, where
minimizing a sparsity-regularized loss on a small set of training instances
guarantees recovery of Bellman--Ford updates~\citep{Nerem2026ShortestPaths}.
Can carefully constructed small instances similarly guarantee that a desired
merge rule is learned and generalizes to larger inputs?  More broadly, it is
exciting to consider which other families of algorithms could guide NN
architectures built around repeated computations, yielding further
data-dependent approaches that retain the efficiency of their classical
counterparts.
\clearpage
\section*{Acknowledgments}
This work was supported in part by the National Science Foundation under grants
CCF-2112665, IIS-2435808, and CNS-2622219.
We thank Andrew B. Kahng and Siavash Mirarab for helpful discussions.

\section*{AI Use Statement}
We used LLM tools to assist with polishing and editing writing and proofs,
experiment implementation, data generation, generating figures, literature
searches, and formatting references. The authors reviewed the AI-assisted work
and take responsibility for the final content.

\bibliographystyle{iclr2027_conference}
\bibliography{references}

\clearpage
\appendix

\section{Proofs}
\label[appendix]{app:proofs}

\begin{proof}[Proof of \cref{thm:equivariance}]
The leaf embeddings are permutation equivariant by the permutation equivariance of $E_\theta$. Next, the initial merge embeddings (for merges between singletons) are permutation equivariant because $\mathrm{MLP}_\mathrm{init}$ has permutation-invariant inputs. Suppose at some step $t$ all merge embeddings are permutation equivariant, so the selected merge is also equivariant. New merge embeddings are computed by $\mathrm{MLP}_\mathrm{update}$, which is invariant to swapping the two merged clusters $A$ and $B$, so the new merge embeddings remain
permutation equivariant. Induction over merge steps gives
$\mathcal{T}_\theta(\pi X)=\pi \mathcal{T}_\theta(X)$.
\end{proof}

\begin{proof}[Proof of \cref{lem:running_time}]
Initialization creates $\binom n2$ queue entries. Each of the $n-1$ merges inserts at most $n-1$ new entries onto the priority queue, for a total of $O(n^2)$ insertions. 
Each insertion requires one evaluation of $\mathrm{MLP}_{\mathrm{init}}$ or
$\mathrm{MLP}_{\mathrm{update}}$ and one score computation. For fixed network
size and embedding dimension, each takes constant time, so these computations
cost $O(n^2)$ in total.
The number of removals from the queue is also $O(n^2)$ since every inserted candidate is
popped at most once. Thus there are $O(n^2)$ queue operations, each costing
$O(\log n)$. Adding the $O(n^2)$ neural computation cost and the assumed
$O(n^2)$ encoder cost gives a total running time of $O(n^2\log n)$.
\end{proof}

\begin{proof}[Proof of \cref{thm:ranking_margin_recovery}]
Initially, every cluster is a singleton in $\mathcal{T}^\star$. If
$\cC_t\subseteq\mathcal C(\mathcal{T}^\star)$ and $|\cC_t|>1$, contracting each current
subtree to a leaf gives a full binary tree with a pair of sibling leaves.
Their corresponding clusters form a positive pair. The margin condition
therefore makes every highest-scoring pair positive, so the next merge again
forms a cluster of $\mathcal{T}^\star$. By induction, all $n-1$ merges belong to
$\mathcal{T}^\star$, and the output is $\mathcal{T}^\star$.
\end{proof}

\begin{proof}[Proof of \cref{thm:contextual_linkage_universality}]
View each tree as nested multisets of size two, with each multiset representing an internal node and containing that node's two children.
Using this representation, we first define a continuous injective encoding of trees, then show
how the leaf encoder and merge updates can approximate it and compute the
score. 

\noindent\textbf{Encoding unordered pairs.}
We begin by considering an injective representation of multisets of size two that we will later apply recursively to encode trees. For a Euclidean vector $u\in\R^r$, use the DeepSets feature map~\citep{Zaheer2017DeepSets}
\[
\varphi(u)=\bigl(u_1,u_1^2,\ u_2,u_1u_2,u_2^2,\ \ldots,\
u_r,u_1u_r,\ldots,u_r^2\bigr).
\]
This lists each coordinate followed by its products with earlier coordinates
and itself. The key property is that $\varphi(u)+\varphi(v)$ uniquely determines
the unordered pair $\{u,v\}$. Its linear coordinates give $a=u+v$, and its
quadratic coordinates give $Q=uu^\top+vv^\top$. Thus we know
\[
2Q-aa^\top=(u-v)(u-v)^\top.
\]
This determines $d=u-v$ up to sign, and hence the unordered pair
$\{u,v\}=\{(a+d)/2,(a-d)/2\}$. Therefore the summed encoding is injective
on unordered pairs.

Throughout the recursive construction, we extend vectors
with trailing zeros when needed. The ordering of $\varphi(u)$ makes this
convention consistent, since appending zeros to its input only appends zeros
to its output.

\noindent\textbf{Encoding trees.}
Let $\operatorname{ht}(\tau)$ be the height of $\tau$, where singletons have height zero. Define
\[
\begin{aligned}
z(\tau_{\{x_i\}})&=(0,x_i),\\
z(\{\tau_1,\tau_2\})
&=\left(1+\max\{\operatorname{ht}(\tau_1),\operatorname{ht}(\tau_2)\},\,
\varphi(z(\tau_1))+\varphi(z(\tau_2))\right).
\end{aligned}
\]
We prove injectivity of $z$ by induction on height $\ell$. The first coordinate
distinguishes trees of different heights. The base case, $\ell = 0$, is immediate as for leaf features $a$, we have defined $z(a) = (0,a,0,\dots)$, so $z(a) = z(b)$ iff $a=b$. Now for the inductive step, assume injectivity holds for heights $\ell < \ell'$, for some $\ell' > 0$. By our construction of the pair encoding $\varphi$, two trees have the same representation, i.e.\ $z(\{\tau_a,\tau_b\}) = z(\{\tau_c,\tau_d\})$, if and only if their child representations are the same $\{z(\tau_a),z(\tau_b)\} = \{z(\tau_c),z(\tau_d)\}$, which by the inductive hypothesis occurs if and only if $\{\tau_a,\tau_b\} = \{\tau_c,\tau_d\}$ (up to tree isomorphism). This proves injectivity for trees of height $\ell = \ell'$. 
Continuity follows as $z$ is a composition of continuous maps. 

\noindent\textbf{Constructing the update.}
We now show how the neural maps can approximate this encoding. 
We
write
\[
z_{A,B}=z(\{\tau_A,\tau_B\}),
\qquad
z^{\mathrm{sum}}_{A,B}=\varphi(z(\tau_A))+\varphi(z(\tau_B)),
\] %
so $z^{\mathrm{sum}}_{A,B}$ is the second part of $z_{A,B}$ following the height
coordinate. Thus $z_{A,B}$ encodes the tree that would result from merging
$\tau_A$ and $\tau_B$.
These denote exact encodings, which we will show the learned merge embeddings $\psi_{A,B}$ can approximate.
We construct a continuous update map $f$ satisfying
\[
f\bigl(z_{A,C}+z_{B,C},\,z_{A,B}\bigr)=z_{A\cup B,C},
\]
and a continuous readout $\rho$ satisfying
\[
\rho(z_{A,B})=s(\tau_A,\tau_B).
\]
Each merge embedding will approximate $(z_{A,B},\rho(z_{A,B}))$, with $w$
selecting the final score coordinate.

We choose a sufficiently wide $\mathrm{MLP}_{\mathrm{update}}$ to approximate
the continuous map $(f,\rho\circ f)$ using the encoding portions of
$\psi_{A,C}+\psi_{B,C}$ and $\psi_{A,B}$. It ignores the score and
absolute-difference coordinates.

Now we show the approximability of $z$ by induction. For the base case, consider singleton pairs. The leaf encoder $E_\theta$ applies $\mathrm{MLP}_{\mathrm{leaf}}$ to each leaf feature
to approximate $x_i\mapsto\varphi(z(\tau_{\{x_i\}}))$. Adding the two exact leaf
encodings and prepending height one gives $z_{\{x_i\},\{x_j\}}$. 
  Thus
$\mathrm{MLP}_{\mathrm{init}}$ can directly approximate this encoding and
its score under $\rho$.

For the inductive step, suppose the tree parts of the current merge embeddings
approximate $z_{A,B}$ for every pair. When $A$ and $B$ merge,
\cref{eq:update} gives
\[
\color{black}
\psi_{A\cup B,C}=\mathrm{MLP}_{\mathrm{update}}\!\left(
[\psi_{A,C}+\psi_{B,C},\,|\psi_{A,C}-\psi_{B,C}|,\,\psi_{A,B}]
\right).
\]
We first show how $z_{A,C}+z_{B,C}$ and $z_{A,B}$ determine the new encoding
$z_{A\cup B,C}$. These are exact quantities, while the update MLP receives
approximations to them through the merge embeddings.
The tree coordinates of the third input approximate
$z_{A,B}=z(\{\tau_A,\tau_B\})$, which is exactly
$z(\tau_{A\cup B})$ after the merge. To construct the new encoding, we also need
$z(\tau_C)$. The sum $z^{\mathrm{sum}}_{A,C}+z^{\mathrm{sum}}_{B,C}$ contains
one copy each of $\varphi(z(\tau_A))$ and $\varphi(z(\tau_B))$, and two
copies of $\varphi(z(\tau_C))$. Subtracting $z^{\mathrm{sum}}_{A,B}$ cancels
the contributions of $A$ and $B$, giving
\[
\varphi(z(\tau_C))
=\frac{z^{\mathrm{sum}}_{A,C}+z^{\mathrm{sum}}_{B,C}
-z^{\mathrm{sum}}_{A,B}}{2}.
\]
The linear terms of $\varphi(z(\tau_C))$ give the coordinates
of $z(\tau_C)$, including its height coordinate. We now have both child encodings,
so the recursion gives
\[
z_{A\cup B,C}
=\left(1+\max\{\operatorname{ht}(\tau_{A\cup B}),\operatorname{ht}(\tau_C)\},\,
\varphi(z_{A,B})+\varphi(z(\tau_C))\right).
\]
This defines $f$ as a continuous function of
$z_{A,C}+z_{B,C}$ and $z_{A,B}$. \cref{eq:update} supplies
approximations to these quantities in its tree coordinates.

Since trees with at most $N$ leaves have height at most $N-1$,
their encodings use only finitely many coordinates. Choose a common bound $D$ for the number of coordinates, and extend shorter encodings with trailing zeros so that all encodings $z_{A,B}$ have the same dimension. 
Since there are at most $N-1$ merge steps, uniform continuity lets us choose
the MLP approximations so that the resulting encodings have arbitrarily small
error throughout the run.

Since $z$ is continuous and injective on a compact domain, it is an embedding (a homeomorphism onto its image). 
Thus the continuous invariant score can be written as
\[
s(\tau_A,\tau_B)=\rho(z_{A,B})
\]
for a continuous readout $\rho$. Since the model can uniformly approximate
$z_{A,B}$ and the MLP readout can uniformly approximate $\rho$, uniform
continuity allows us to combine these approximations, giving universality.

\noindent\textbf{Global context.}
For globally dependent scores, we attach the same encoding of the entire input
to every leaf and reuse the preceding construction.
For the transformer case, use the DeepSets representation
\[
\gamma(x)=\bigl(x_1^{\alpha_1}\cdots x_p^{\alpha_p}\bigr)_{
\alpha\in\mathbb N_0^p,\,\alpha_1+\cdots+\alpha_p\le N},
\qquad
\Gamma(X)=\sum_{i=1}^n\gamma(x_i).
\]
These summed monomials identify $X$ up to permutation, including its size
\citep[Thm.~2.12]{Alberti2023Sumformer}. Replace $x_i$ by
$y_i=(x_i,\Gamma(X))$ in the construction of $z$, choosing $D$ for these larger
leaf features. The required leaf map is
$E_\theta(X)_i\approx\varphi((0,y_i))$. A transformer approximates this map
uniformly using pointwise maps and uniform attention, with $n$ converting
averages into sums~\citep[Thm.~4.3]{Alberti2023Sumformer}. The same proof
then applies to $s^X(\tau_A,\tau_B)$.
\end{proof}

\begin{proof}[Proof of \cref{thm:lw}]
Assume scores for merges between singletons are initialized exactly. For constant coefficients
$\alpha_A=\alpha_B=\alpha$, set \textcolor{black}{$\psi_{A,B}=s(\tau_A,\tau_B)$} and $w=1$.
A rectified linear unit (ReLU) MLP can exactly compute the linear combination of update inputs
\[
\color{black}
\begin{aligned}
s(\tau_{A\cup B},\tau_C)={}&
\alpha\bigl(s(\tau_A,\tau_C)+s(\tau_B,\tau_C)\bigr)\\
&+\beta s(\tau_A,\tau_B)
+\gamma\left|s(\tau_A,\tau_C)-s(\tau_B,\tau_C)\right|.
\end{aligned}
\]
Thus the same parameters and one coordinate realize the recurrence exactly at
every input size. \cref{tab:lw_params} gives the coefficients for the named
linkages under our score convention.
\end{proof}

\begin{proof}[Proof of \cref{cor:robust-sl}]
The radius $r_k(x_i)$ is the $k$th order statistic of the continuous distances
$\{\|x_i-x_j\|:j\ne i\}$, hence is continuous in $X$, including at ties.
The robust SL score
$s^X(\tau_A,\tau_B)=-\min_{i\in A,j\in B}\lambda_{ij}$ is therefore continuous.
It is also invariant to input permutations, tree isomorphism, and swapping
$A$ and $B$. The transformer case of
\cref{thm:contextual_linkage_universality} therefore applies.
\end{proof}

\FloatBarrier
\section{LW Parameters}
\label[appendix]{app:lw_parameters}

\begin{table}[H]
\centering
\setlength{\tabcolsep}{4pt}
\begin{tabular*}{\linewidth}{@{\extracolsep{\fill}}p{0.34\linewidth}cccc@{}}
\toprule
\textbf{Method} & $\alpha_A$ & $\alpha_B$ & $\beta$ & $\gamma$ \\
\midrule
SL
& $\tfrac{1}{2}$ & $\tfrac{1}{2}$ & $0$ & $\tfrac{1}{2}$ \\[4pt]
CL
& $\tfrac{1}{2}$ & $\tfrac{1}{2}$ & $0$ & $-\tfrac{1}{2}$ \\[4pt]
Group average
& $\tfrac{n_A}{n_A+n_B}$ & $\tfrac{n_B}{n_A+n_B}$ & $0$ & $0$ \\[4pt]
WPGMA
& $\tfrac{1}{2}$ & $\tfrac{1}{2}$ & $0$ & $0$ \\[4pt]
Centroid
& $\tfrac{n_A}{n_A+n_B}$ & $\tfrac{n_B}{n_A+n_B}$
& $-\tfrac{n_A n_B}{(n_A+n_B)^2}$ & $0$ \\[4pt]
Median
& $\tfrac{1}{2}$ & $\tfrac{1}{2}$ & $-\tfrac{1}{4}$ & $0$ \\[4pt]
Minimum variance (Ward's)
& $\tfrac{n_A+n_C}{n_A+n_B+n_C}$
& $\tfrac{n_B+n_C}{n_A+n_B+n_C}$
& $-\tfrac{n_C}{n_A+n_B+n_C}$ & $0$ \\[4pt]
\bottomrule
\end{tabular*}
\caption{Parameter choices for common LW HC
methods. Scores are negative dissimilarities, using squared dissimilarities for
centroid, median, and Ward. See \citet{Milligan1979} for discussion.}
\label{tab:lw_params}
\end{table}

\FloatBarrier
\section{Experiment Details}
\label[appendix]{app:experiment_details}
\raggedbottom

\paragraph{Training Presentations.}
For clock-tree and phylogenetic training, we define a \emph{training presentation}
as one loss-bearing traversal of the merge trajectory for one instance.
Learning-rate schedules count presentations, rather than optimizer steps,
because each optimizer update accumulates gradients over multiple presentations.

\paragraph{Model Configuration.}
\cref{tab:experiment-config} summarizes the clock-tree and phylogenetic models. 
The rings-and-disks configuration is given in \cref{app:rings-disks}.
Both merge MLPs have one ReLU hidden layer and a linear output; the initializer
uses \cref{eq:init_pair}. In addition to the three merge features in
\cref{eq:update}, the update receives
\[
 [\bar h_{A\cup B},\ \bar h_C,\ \bar h_{X\setminus(A\cup B\cup C)}],\quad\text{where}
 \qquad \bar h_D=|D|^{-1}\sum_{i\in D}h_i.
\]
 These additional features preserve
invariance to exchanging $A$ and $B$.

\begin{table}[H]
\centering
\small
\begin{tabular*}{\linewidth}{@{\extracolsep{\fill}}lll@{}}
\toprule
Setting & Clock tree & Phylogenetics \\
\midrule
Encoder & Transformer & NeuralNJ axial MSA encoder \\
Layers / attention heads & 2 / 4 & 6 / 8 \\
Encoder width / feed-forward width & 44 / 176 & 64 / 256 \\
Encoder activation / dropout & ReLU / 0 & GELU / 0.4 \\
Leaf and merge dimension $d$ & 11 & 64 \\
Merge MLP hidden width & 44 & 128 \\
Total trainable parameters & 53,196 & 487,360 \\
Optimizer / weight decay & AdamW / $10^{-4}$ & Adam / 0 \\
Trees per gradient accumulation & 100 & 10 \\
Forward batch size & 100 & 4, 2, or 1 (see text) \\
Initial / minimum learning rate & $5\!\times\!10^{-4}$ / $10^{-5}$ & $5\!\times\!10^{-4}$ / $10^{-5}$ \\
Cosine horizon (tree presentations) & 500,000, then constant & 350,000 \\
Training budget (presentations) & 1,500,000 & 216,000 \\
Selected checkpoint (presentations) & 1,300,000 & 196,000 \\
Time to selected checkpoint & $\approx26$ hours & $\approx2.8$ days \\
\bottomrule
\end{tabular*}
\caption{Architecture and optimization settings for the reported clock-tree
and phylogenetic checkpoints. Times include validation.}
\label{tab:experiment-config}
\end{table}

\paragraph{Loss and Optimization.}
Clock-tree training compares the highest-scoring compatible positive against
at most 50 highest-scoring margin-violating negatives.
Phylogenetic training uses all compatible positives and all violating comparisons.
Each tree's hinge loss is normalized by its number of retained strictly
violating comparisons across all merge steps:
\[
 \widehat{\mathcal L}(\mathcal T)=
 \frac{\sum_{(t,p,q)\in V_{\mathcal T}}[m+s_\theta(q)-s_\theta(p)]_+}{|V_{\mathcal T}|},
 \qquad m=0.5,
\]
where $V_{\mathcal T}$ contains the retained comparisons with
$m+s_\theta(q)-s_\theta(p)>0$. We average tree losses within each accumulation window,
use default Adam momentum parameters, and apply no gradient clipping.
The cosine schedules are specified in \cref{tab:experiment-config}.
 Phylogenetic forward
batches group equal taxon and site counts, with up to four trees for
$n\le45$ and two for $46\le n\le70$;
larger trees are processed sequentially.
\cref{fig:training-validation-curves} shows the validation histories
and the selected checkpoints.

\begin{figure}[htbp]
\centering
\includegraphics[width=\linewidth,trim={6.182bp 6.182bp 5.635bp 6.194bp},clip]{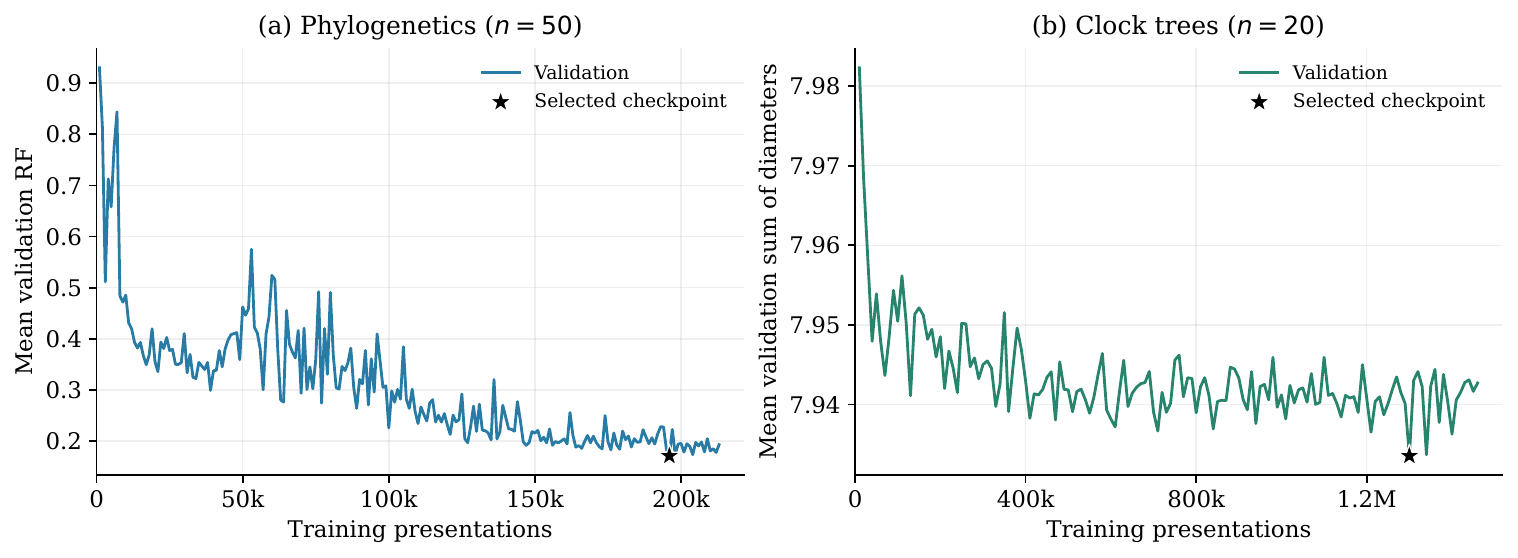}
\caption{Validation curves for the phylogenetic and clock-tree models. 
(a) Mean RF on 90 held-out 50-taxon alignments per evaluation.
(b) Mean sum of $L_1$ diameters on 250 held-out 20-point sets.
Stars mark the selected checkpoints at 196,000 and 1,300,000 training
presentations, respectively.}
\label{fig:training-validation-curves}
\end{figure}

\FloatBarrier
\subsection{Rings and Disks}
\label[appendix]{app:rings-disks}

We design instances such that each 120-point cloud has 30 points per cluster, where they are equally spaced points on
concentric rings of radii 0.78 and 1, and uniform points in disks of radius
0.35 whose centers are 0.72 apart. The ring center and disk-pair midpoint
are 5 units apart at a random angle; we randomize the ring phase, translate
by a uniform vector in $[-2,2]^2$, and shuffle the points. The target trajectory that we train to replicate first builds the ring
subtree using SL, then the disk subtree using CL, then joins the motifs.

\paragraph{Screened Distribution.}
During data generation, we retain an instance only if SL on its ring points and CL on its disk
points each permit perfect two-subtree recovery, and the combined target
hierarchy permits perfect four-subtree recovery. This rejection sampling
defines the training, validation, and test distribution, with a training
acceptance rate of 41.3\%.

\paragraph{Model, Training, and Selection.}
The point encoder maps each two-dimensional point to 140 features, applies two four-head transformer layers (feed-forward width 560, ReLU, zero dropout, and no positional encoding), and then maps each output to a 35-dimensional embedding. The initializer and update MLPs have input widths 70 and 210, respectively, each with hidden width 140 and output width 35. The update MLP includes the set average features defined above. Including the linear ranker, the model has 528,780 parameters.
We train from scratch on 500,000 screened 120-point clouds, sampling batches
of 16 with replacement. We treat the merge selected by the target trajectory
as the sole positive ($m=0.5$); all other active pairs, including other
target-compatible merges, are eligible negatives. At each merge we retain up
to 120 highest-scoring strictly violating negatives and normalize each
tree's hinge sum by its retained-comparison count, then average over
loss-bearing trees in the batch. Adam uses learning rate $3\!\times\!10^{-4}$,
zero weight decay, gradient-norm clipping at 5, and a cosine schedule with
minimum $6\!\times\!10^{-6}$ and horizon 30,000 updates. Validation on 500 independent clouds
every 1,000 updates selects the minimum mean pruning error, attained at update 2,000
(32,000 sampled trees).

\paragraph{Evaluation.}
We use 100 independent clouds at each of $n=120,180,240$, with equal numbers
of points in each cluster and the same geometry and screening. We define error as the
minimum point misclassification rate over four-subtree prunings of the
inferred hierarchy, with a bijection from subtrees to ground-truth labels.
The convex combination of linkage functions baseline uses
$(1-\alpha)\mathrm{MIN}+\alpha\mathrm{MAX}$,
with one shared $\alpha\in\{0,0.05,\ldots,1\}$ per size minimizing mean error
on that size's evaluation set. We also report SL and CL, the endpoints.

\Cref{tab:rings-disks-results,fig:rings-disks-size-generalization} summarize
classification error. The selected checkpoint perfectly recovers all four
clusters on 91, 85, and 77 of the 100 clouds at each size, respectively.

\begin{table}[htbp]
\centering
\small
\begin{tabular}{@{}rrrrrr@{}}
\toprule
$n$ & \nnlinkage{} & Best $\alpha$-linkage & $\alpha$ & SL & CL \\
\midrule
120 & $0.40\pm0.16$ & $14.82\pm1.06$ & 0.00 & $14.82\pm1.06$ & $24.58\pm0.05$ \\
180 & $0.46\pm0.13$ & $2.13\pm0.56$ & 0.10 & $12.34\pm1.13$ & $24.65\pm0.04$ \\
240 & $0.61\pm0.12$ & $0.95\pm0.35$ & 0.10 & $14.14\pm1.15$ & $24.73\pm0.03$ \\
\bottomrule
\end{tabular}
\caption{Rings-and-disks point misclassification (\%), mean $\pm$ one standard
error over 100 held-out clouds per size. Lower is better. The baseline uses
one evaluation-set-tuned $\alpha$ per size.}
\label{tab:rings-disks-results}
\end{table}

\paragraph{Local Merge Behavior.}
On the 100 held-out $n=120$ instances, we compare the model's chosen merges
with SL on rings and CL on disks, using the reference linkage's objective:
minimum intercluster Euclidean distance for SL and maximum intercluster
Euclidean distance for CL. At each pure same-label merge along the model's
trajectory, let $d_{\mathrm{model}}$ be the chosen pair's objective and
$d_{\mathrm{ref}}$ the smallest objective among active same-label pairs in
that motif. Averaging $100|d_{\mathrm{model}}-d_{\mathrm{ref}}|/d_{\mathrm{ref}}$
over these merges gives \textbf{0.010\% for rings (SL)} and
\textbf{0.485\% for disks (CL)}.

\begin{figure}[htbp]
\centering
\includegraphics[width=0.75\linewidth]{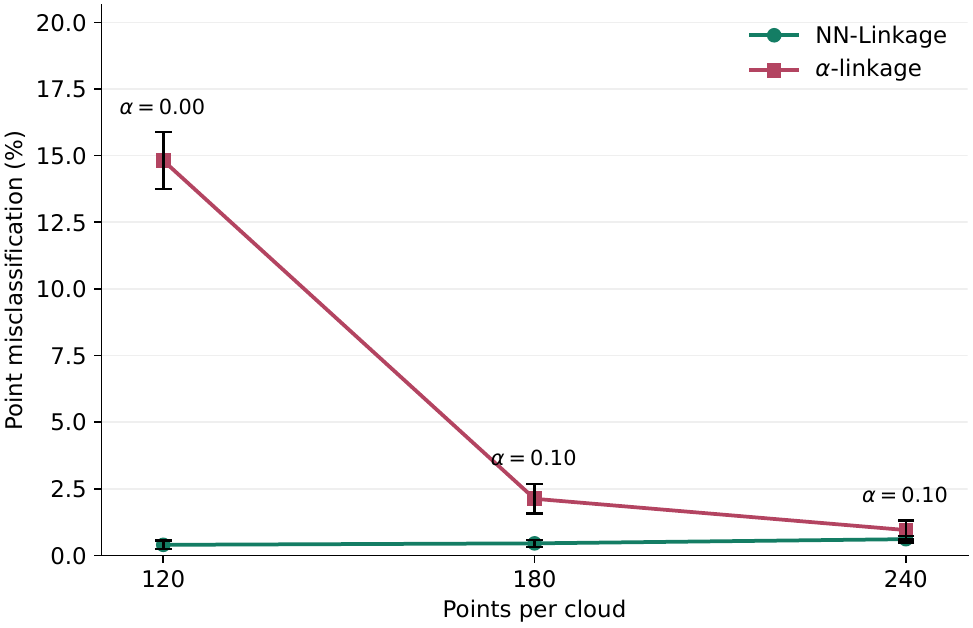}
\caption{Rings-and-disks classification error at 120, 180, and 240 total
points, evaluated over 100 held-out instances at each size; error bars show
one standard error. Annotations give the evaluation-set-tuned mixing
parameter for the $\alpha$-linkage baseline. Lower is better.}
\label{fig:rings-disks-size-generalization}
\end{figure}

\FloatBarrier
\subsection{Clock Tree Routing}

\paragraph{Reduction to Sum of Diameters.}
\label[appendix]{app:clock-tree-reduction}
Under the linear delay model, optimal deferred-merge embedding (DME) of a
fixed topology has wire length
$\tfrac12\sum_{A\in\cC(\mathcal T)}\operatorname{diam}_1(A)
+\tfrac12\operatorname{diam}_1(X)$~\citep{Chao1992DME,ChenYoung2019DimSum}.
The final term is independent of topology so can be removed, giving the diameter-sum objective
used in the main text.

\paragraph{Training and Selection.}
The encoder linearly maps each two-dimensional point to 44 features, applies the transformer blocks in \cref{tab:experiment-config} without positional embeddings, and then linearly maps each output to an 11-dimensional embedding. We select the checkpoint with the lowest mean diameter sum on 250 held-out 20-point validation sets. \cref{tab:uniform-sizegen} gives the size-generalization results underlying \cref{fig:uniform-sizegen}.

\paragraph{VLSI Protocol.}
All 20 BST-DME and ISPD 2009/2010 instances are listed in
\cref{tab:vlsi-instances} along with the performance of CL and \nnlinkage{}. Each subtree's coordinates are rescaled
independently along each axis to $[0,1]$ for model input; objective evaluation uses the original coordinates. The recursive replacements
described in the main text use breadth-first traversal of the current tree,
including accepted replacements.

For CL+$k$, we perform one bottom-up refinement pass: at each eligible node,
repeatedly split the largest-diameter selected subtree until there are $k$
disjoint subtrees, then optimize their interconnection with the DimSum dynamic
program. Model+$k$ applies this pass to each proposed replacement before
acceptance. \cref{tab:vlsi-instances,tab:vlsi-pareto} report per-instance
results and aggregate quality--runtime comparisons.

\begin{table}[H]
\centering
\begin{tabular*}{\linewidth}{@{\extracolsep{\fill}}rrrrr@{}}
\toprule
\textbf{$n$} & \textbf{Clouds} & \textbf{Mean \% vs CL} & \textbf{Median \% vs CL} & \textbf{Std. dev.} \\
\midrule
20 & 50 & -0.541 & 0.000 & 1.672 \\
30 & 50 & -0.601 & -0.038 & 1.261 \\
40 & 50 & -0.294 & -0.195 & 1.508 \\
50 & 50 & -0.468 & -0.209 & 1.472 \\
100 & 50 & -0.265 & -0.046 & 1.136 \\
250 & 50 & -0.392 & -0.396 & 0.980 \\
500 & 50 & -0.117 & -0.185 & 0.654 \\
750 & 50 & 0.147 & 0.117 & 0.709 \\
1000 & 50 & 0.494 & 0.406 & 0.669
\\
\bottomrule
\end{tabular*}
\caption{Uniform-point size generalization without recursive replacement
or local refinement. Statistics
summarize per-cloud percentage objective changes versus CL; negative is better.}
\label{tab:uniform-sizegen}
\end{table}

\begin{table}[H]
\centering
\setlength{\tabcolsep}{3pt}
\begin{tabular*}{\linewidth}{@{\extracolsep{\fill}}rrrrrr@{}}
\toprule
\textbf{$k$} & \textbf{CL mean \%} & \textbf{CL time (s)} & \textbf{Model mean \%} & \textbf{Model time (s)} & \textbf{Model median \%} \\
\midrule
8 & -1.187 & 1.001 & -2.085 & 5.200 & -1.776 \\
10 & -1.456 & 1.242 & -2.231 & 5.636 & -1.999 \\
12 & -1.726 & 1.483 & -2.402 & 6.282 & -2.277 \\
14 & -1.957 & 2.045 & -2.599 & 8.722 & -2.456 \\
16 & -2.109 & 4.768 & -2.598 & 22.885 & -2.437 \\
17 & -2.114 & 10.387 & -- & -- & -- \\
18 & -2.177 & 25.022 & -- & -- & --
\\
\bottomrule
\end{tabular*}
\caption{VLSI quality and runtime comparison. Objective changes and runtimes
are averaged over the same 20 instances.
One timed pass per instance uses one CPU thread and an RTX A6000 GPU for
\nnlinkage{}. Times include CL construction and all applicable refinement and
recursive replacement, but exclude parsing, model loading, and startup warmup.}
\label{tab:vlsi-pareto}
\end{table}

\begin{table}[p]
\centering
\footnotesize
\setlength{\tabcolsep}{3pt}
\textbf{CL with refinement}\\[3pt]
\begin{tabular*}{\linewidth}{@{\extracolsep{\fill}}lrrrrrrrr@{}}
\toprule
Instance & Sinks & $k=8$ & $10$ & $12$ & $14$ & $16$ & $17$ & $18$ \\
\midrule
s1423 & 74 & -0.889 & -0.985 & -1.135 & -1.435 & -1.373 & -1.321 & -1.545 \\
s1r1 & 81 & -0.413 & -1.612 & -2.216 & -2.227 & -2.227 & -2.424 & -2.424 \\
s2r1 & 88 & -2.784 & -2.980 & -3.144 & -3.441 & -3.539 & -3.539 & -3.539 \\
s3r1 & 131 & -1.055 & -1.843 & -2.128 & -2.472 & -2.264 & -1.878 & -1.893 \\
s5378 & 179 & -0.956 & -0.977 & -1.033 & -1.037 & -1.679 & -1.679 & -1.710 \\
r1 & 267 & -1.038 & -1.660 & -2.754 & -3.197 & -3.275 & -3.282 & -3.296 \\
p1 & 269 & -1.509 & -1.615 & -1.775 & -2.352 & -3.120 & -3.277 & -3.473 \\
s15850 & 597 & -0.835 & -0.960 & -1.015 & -1.029 & -1.160 & -1.088 & -1.150 \\
r2 & 598 & -2.193 & -1.794 & -2.518 & -2.674 & -3.327 & -3.311 & -3.322 \\
p2 & 603 & -1.368 & -1.654 & -1.940 & -2.300 & -2.439 & -2.503 & -2.862 \\
s4r3 & 623 & -0.848 & -1.022 & -1.325 & -1.431 & -1.514 & -1.502 & -1.586 \\
r3 & 862 & -0.771 & -1.030 & -1.555 & -1.787 & -1.744 & -1.758 & -1.864 \\
06 & 981 & -1.791 & -2.013 & -2.142 & -2.239 & -2.296 & -2.256 & -2.290 \\
05 & 1016 & -1.360 & -1.877 & -2.083 & -2.252 & -2.352 & -2.349 & -2.366 \\
08 & 1134 & -1.010 & -1.220 & -1.223 & -1.442 & -1.402 & -1.453 & -1.546 \\
03 & 1200 & -0.876 & -1.064 & -1.177 & -1.144 & -1.181 & -1.176 & -1.188 \\
04 & 1845 & -1.110 & -1.260 & -1.463 & -1.814 & -2.122 & -2.253 & -2.252 \\
r4 & 1903 & -1.146 & -1.572 & -1.699 & -1.893 & -2.008 & -1.958 & -1.946 \\
07 & 1915 & -1.114 & -1.169 & -1.224 & -1.634 & -1.736 & -1.792 & -1.762 \\
r5 & 3101 & -0.682 & -0.814 & -0.966 & -1.329 & -1.415 & -1.482 & -1.527 \\
\midrule
Mean &  & -1.187 & -1.456 & -1.726 & -1.957 & -2.109 & -2.114 & -2.177
\\
\bottomrule
\end{tabular*}
\par\medskip
\textbf{Recursive \nnlinkage{} with optional refinement}\\[3pt]
\begin{tabular*}{\linewidth}{@{\extracolsep{\fill}}lrrrrrrr@{}}
\toprule
Instance & Sinks & No refinement & $k=8$ & $10$ & $12$ & $14$ & $16$ \\
\midrule
s1423 & 74 & 0.000 & -1.329 & -1.373 & -1.667 & -1.627 & -1.373 \\
s1r1 & 81 & -1.133 & -2.726 & -2.031 & -2.635 & -3.050 & -2.696 \\
s2r1 & 88 & -1.238 & -3.060 & -3.803 & -4.388 & -4.361 & -4.191 \\
s3r1 & 131 & -0.939 & -1.413 & -1.968 & -2.339 & -2.789 & -2.474 \\
s5378 & 179 & -0.205 & -1.794 & -1.791 & -1.753 & -1.739 & -1.690 \\
r1 & 267 & -2.712 & -3.211 & -3.351 & -3.369 & -3.762 & -4.035 \\
p1 & 269 & -3.183 & -3.657 & -3.998 & -4.021 & -4.468 & -4.433 \\
s15850 & 597 & -0.089 & -1.489 & -1.601 & -1.630 & -1.506 & -1.561 \\
r2 & 598 & -1.634 & -2.961 & -2.788 & -3.483 & -3.570 & -3.835 \\
p2 & 603 & -1.498 & -2.739 & -2.972 & -2.549 & -2.942 & -3.008 \\
s4r3 & 623 & -0.358 & -1.427 & -1.421 & -1.518 & -1.789 & -1.680 \\
r3 & 862 & -0.452 & -1.819 & -2.048 & -2.351 & -2.506 & -2.399 \\
06 & 981 & -0.625 & -2.873 & -2.964 & -2.829 & -3.143 & -3.187 \\
05 & 1016 & -0.481 & -1.820 & -2.259 & -2.359 & -2.369 & -2.378 \\
08 & 1134 & -0.472 & -1.719 & -1.741 & -1.585 & -1.910 & -1.881 \\
03 & 1200 & -0.378 & -1.255 & -1.227 & -1.256 & -1.286 & -1.294 \\
04 & 1845 & -0.330 & -1.745 & -1.897 & -2.205 & -2.591 & -2.920 \\
r4 & 1903 & -0.233 & -1.678 & -2.055 & -2.215 & -2.406 & -2.510 \\
07 & 1915 & -0.452 & -1.757 & -1.964 & -2.083 & -2.305 & -2.401 \\
r5 & 3101 & -0.181 & -1.230 & -1.362 & -1.800 & -1.857 & -2.017 \\
\midrule
Mean &  & -0.829 & -2.085 & -2.231 & -2.402 & -2.599 & -2.598
\\
\bottomrule
\end{tabular*}
\caption{Per-instance diameter-sum change (\%) relative to raw CL, negative is better. }
\label{tab:vlsi-instances}
\end{table}

\FloatBarrier
\subsection{Phylogenetic Reconstruction}

\paragraph{Data and Checkpoint.}
Synthetic alignments follow NeuralNJ's AliSim/GTR+I+G generation settings
\citep{Zhang2025NeuralNJ,LyTrong2022AliSim}. We use 37,349 alignments with counts proportional to
$\exp(-(n-50)^2/(2\cdot15^2))$ for $20\le n\le100$, with lengths
128--1024 in increments of 32, approximately balanced within each choice of $n$.
We select the checkpoint with the lowest mean 50-taxon validation RF.
The synthetic test grid contains 128 independently generated alignments per
setting, disjoint from validation and shared by all four methods.

\paragraph{Empirical Coverage and References.}
We use the original \texttt{Best\_observed} likelihood trees from
\citet{Zhou2018PhylogenomicBenchmarks} as references.
We exclude alignments exceeding NeuralNJ's reported limits of 17,000 sites
for JarvD5a, SongD1, and TarvD7, and 4,400 for WickD3b; retained alignments
are not truncated. Both neural models were attempted on all 26,177 eligible
alignments. \nnlinkage{} ran out of GPU memory on 12 and NeuralNJ on 41 instances.
All such instances were in JarvD5a.
\Cref{fig:phylo-real-pareto,tab:phylo-real-neural-full} report the 26,136
alignments successfully processed by both models, using identical reference
trees and pooled means that weight alignments equally.
For the separate four-method comparison in \cref{tab:phylo-real-full},
we sample 250 shared successes per dataset uniformly without replacement
and use the same 1,000 inputs for every method.

\begin{table}[H]
\centering
\small
\begin{tabular}{@{}lrrr@{}}
\toprule
Dataset & Taxa & Median sites & Site range \\
\midrule
JarvD5a & 38--48 & 1,522 & 66--16,942 \\
SongD1 & 37 & 2,588 & 486--14,659 \\
TarvD7 & 10--36 & 1,749 & 288--16,812 \\
WickD3b & 51--97 & 887 & 278--3,868 \\
\bottomrule
\end{tabular}
\caption{Dimensions of the 26,136 matched empirical alignments.}
\label{tab:phylo-real-dimensions}
\end{table}

\paragraph{Encoder and Sitewise Scoring.}
The leaf encoder is taken directly from NeuralNJ's released implementation. It maps each four-channel nucleotide vector to 64 features, applies a second 64-feature MLP layer, and then applies the axial transformer blocks in \cref{tab:experiment-config}, with attention across sites and taxa, and no positional embeddings. A/C/G/T use four-channel one-hot vectors, where gaps and ambiguity symbols use the all-ones vector. 
We retain a separate leaf and merge embedding at each site, apply the shared initializer and update MLPs independently at each site, and sum scores over valid sites:
$s_\theta^X(\tau_A,\tau_B)=\sum_\ell w^\top\psi_{A,B,\ell}$.
The additional leaf averages defined above are also computed separately at each site. No further attention across sites is applied during merging. 

\paragraph{Baselines.}
NeuralNJ uses the released GTR+I+G checkpoint with greedy decoding.
BIONJ is IQ-TREE's distance-based starting
tree under GTR, with topology search disabled. For RAxML-NG, we take the
binary (version 1.2.0) and search settings directly from NeuralNJ's released
code~\citep{Zhang2025NeuralNJ}: GTR+G, ten random and ten parsimony starting
trees, seed 7788, and \texttt{--opt-model off}.

\paragraph{Inference Timing.}
\label[appendix]{app:phylo-timing}
All methods are timed locally. Runtime includes input preparation and tree
construction, excluding model loading and RF scoring; empirical timings
start before file input and end after Newick construction. Both neural
models use an RTX A6000 with 32-bit floating-point (FP32) precision and
process one alignment at a time. BIONJ uses one CPU thread; RAxML-NG uses
eight CPU threads for both synthetic and empirical data.

\begin{table}[tbp]
\centering
\small
\begin{tabular*}{\linewidth}{@{\extracolsep{\fill}}lrrrrr@{}}
\toprule
& & \multicolumn{2}{c}{\nnlinkage{}} & \multicolumn{2}{c}{NeuralNJ} \\
\cmidrule(lr){3-4}\cmidrule(lr){5-6}
Dataset & Trees & RF & Time (s) & RF & Time (s) \\
\midrule
JarvD5a & 14123 & 0.7203 & 0.418 & 0.7070 & 2.052 \\
SongD1 & 423 & 0.3972 & 0.418 & 0.4371 & 2.461 \\
TarvD7 & 11166 & 0.4751 & 0.240 & 0.5259 & 1.251 \\
WickD3b & 424 & 0.5414 & 0.251 & 0.5094 & 1.125 \\
Overall & 26136 & 0.6075 & 0.339 & 0.6221 & 1.701
\\
\bottomrule
\end{tabular*}
\caption{Full  neural-method evaluation underlying
\cref{fig:phylo-real-pareto}: mean normalized RF and runtime on 26,136
shared successful alignments. Overall pools individual alignments.}
\label{tab:phylo-real-neural-full}
\end{table}

\begin{figure}[H]
\centering
\includegraphics[width=\linewidth,trim={9.620bp 7.982bp 2.402bp 13.922bp},clip]{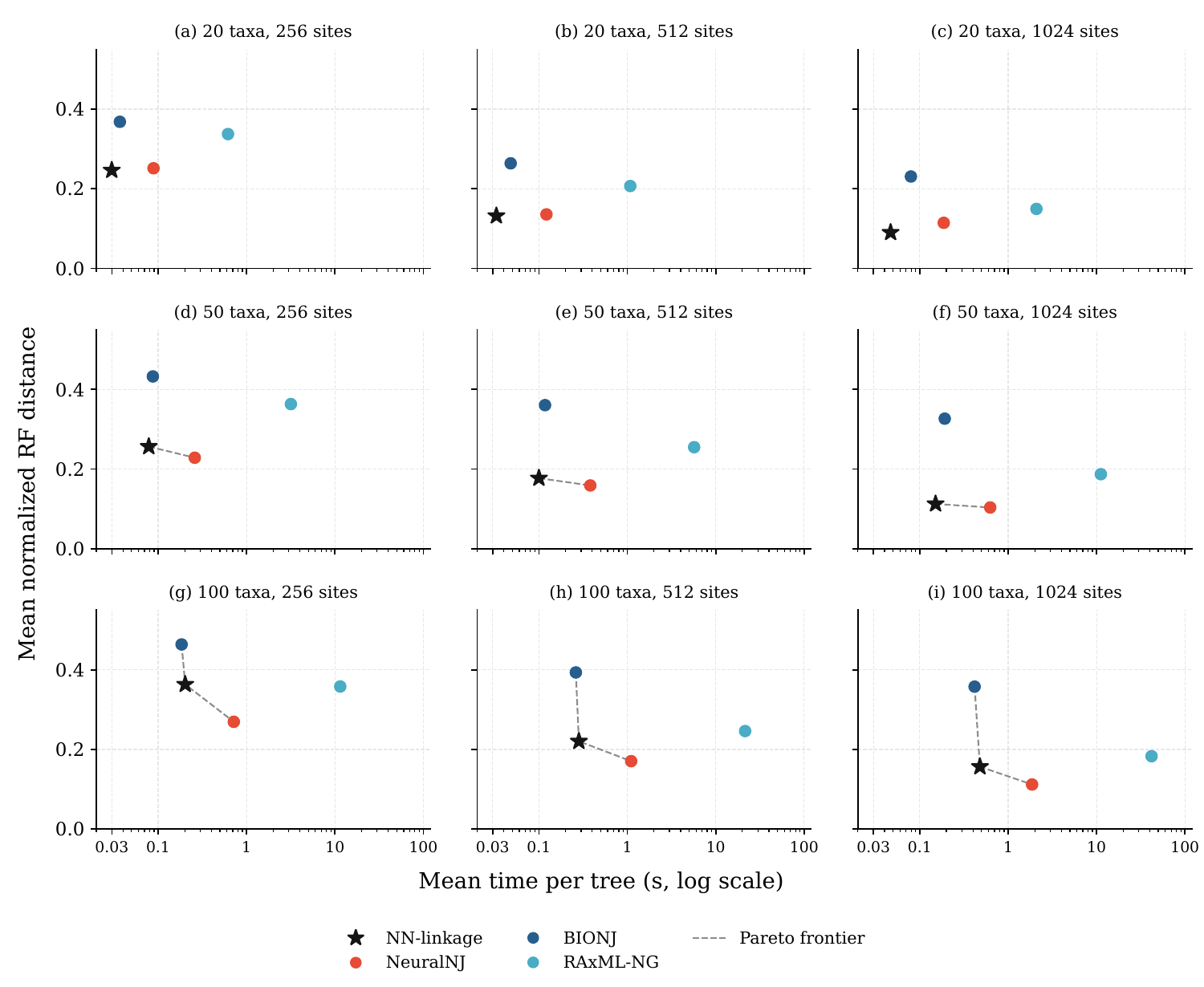}
\caption{Synthetic phylogenetic quality--runtime panels for all combinations of
20, 50, and 100 taxa with 256, 512, and 1024 sites. Each point averages the same
128 held-out instances within a panel. Dashed lines join nondominated points;
numerical results are given in \cref{tab:phylo-synthetic-full}.}
\label{fig:phylo-synthetic-3x3-pareto}
\end{figure}

\begin{table}[tbp]
\centering
\small
\setlength{\tabcolsep}{3pt}
\begin{tabular*}{\linewidth}{@{\extracolsep{\fill}}lrrrrrrrrr@{}}
\toprule
& & \multicolumn{2}{c}{\nnlinkage{}} & \multicolumn{2}{c}{NeuralNJ}
& \multicolumn{2}{c}{BIONJ} & \multicolumn{2}{c}{RAxML-NG} \\
\cmidrule(lr){3-4}\cmidrule(lr){5-6}\cmidrule(lr){7-8}\cmidrule(lr){9-10}
Taxa / sites & Trees & RF & Time (s) & RF & Time (s) & RF & Time (s) & RF & Time (s) \\
\midrule
20 / 256 & 128 & 0.2463 & 0.030 & 0.2514 & 0.089 & 0.3681 & 0.037 & 0.3373 & 0.617 \\
20 / 512 & 128 & 0.1319 & 0.033 & 0.1356 & 0.122 & 0.2638 & 0.048 & 0.2068 & 1.080 \\
20 / 1024 & 128 & 0.0901 & 0.047 & 0.1144 & 0.187 & 0.2307 & 0.080 & 0.1494 & 2.090 \\
50 / 256 & 128 & 0.2565 & 0.078 & 0.2286 & 0.260 & 0.4327 & 0.087 & 0.3632 & 3.174 \\
50 / 512 & 128 & 0.1772 & 0.100 & 0.1589 & 0.381 & 0.3605 & 0.117 & 0.2550 & 5.696 \\
50 / 1024 & 128 & 0.1125 & 0.151 & 0.1036 & 0.627 & 0.3266 & 0.192 & 0.1872 & 11.188 \\
100 / 256 & 128 & 0.3633 & 0.202 & 0.2694 & 0.718 & 0.4637 & 0.184 & 0.3581 & 11.480 \\
100 / 512 & 128 & 0.2204 & 0.283 & 0.1706 & 1.105 & 0.3934 & 0.262 & 0.2461 & 21.496 \\
100 / 1024 & 128 & 0.1567 & 0.480 & 0.1120 & 1.866 & 0.3577 & 0.418 & 0.1830 & 41.882 \\
Overall & 1152 & 0.1950 & 0.156 & 0.1716 & 0.595 & 0.3552 & 0.158 & 0.2540 & 10.967
\\
\bottomrule
\end{tabular*}
\caption{Complete synthetic evaluation: mean normalized RF and runtime on
128 held-out alignments per setting, using the generating trees as references.
Overall pools all 1,152 alignments, equally weighting the nine settings.}
\label{tab:phylo-synthetic-full}
\end{table}

\begin{table}[tbp]
\centering
\small
\setlength{\tabcolsep}{3pt}
\begin{tabular*}{\linewidth}{@{\extracolsep{\fill}}lrrrrrrrrr@{}}
\toprule
& & \multicolumn{2}{c}{\nnlinkage{}} & \multicolumn{2}{c}{NeuralNJ}
& \multicolumn{2}{c}{BIONJ} & \multicolumn{2}{c}{RAxML-NG} \\
\cmidrule(lr){3-4}\cmidrule(lr){5-6}\cmidrule(lr){7-8}\cmidrule(lr){9-10}
Dataset & Trees & RF & Time (s) & RF & Time (s) & RF & Time (s) & RF & Time (s) \\
\midrule
JarvD5a & 250 & 0.7223 & 0.413 & 0.7010 & 2.016 & 0.6750 & 0.229 & 0.4964 & 22.133 \\
SongD1 & 250 & 0.4002 & 0.451 & 0.4411 & 2.651 & 0.4468 & 0.215 & 0.1844 & 15.947 \\
TarvD7 & 250 & 0.4702 & 0.237 & 0.5202 & 1.279 & 0.4867 & 0.155 & 0.2358 & 8.733 \\
WickD3b & 250 & 0.5440 & 0.248 & 0.5109 & 1.098 & 0.5444 & 0.209 & 0.1957 & 18.328 \\
Overall & 1000 & 0.5342 & 0.337 & 0.5433 & 1.761 & 0.5382 & 0.202 & 0.2781 & 16.285
\\
\bottomrule
\end{tabular*}
\caption{Four-method empirical comparison: mean normalized RF and runtime on
250 sampled alignments per dataset, using per-locus best-observed trees as
references. Overall pools all 1,000 alignments, equally weighting the four
datasets.}
\label{tab:phylo-real-full}
\end{table}

\FloatBarrier
\flushbottom
\section{An Approximation Ratio Lower Bound for CL}
\label[appendix]{app:complete-linkage-counterexample}
In this section we give a point set for which CL can yield a tree with diameter sum cost 25\% more than optimal. We then test our trained model on this instance, showing it recovers the optimal tree.

\FloatBarrier
\subsection{Approximation Ratio Lower Bound}

Consider four points $X=\{A,B,C,D\}\subset[0,1]^2$ where
\[
A=(0,0.25),\quad B=(0,0.75),\quad C=(0.25,0),\quad D=(0.25,1).
\]
Let $d$ be the $L_1$ distance in the plane.

\begin{lemma}
Let $\mathrm{OPT}(X)$ be the diameter sum of the optimal tree over $X$.
On $X$, CL can produce a tree with diameter sum
\begin{equation}
\frac{5}{4}\mathrm{OPT}(X).
\end{equation}
\end{lemma}

\begin{proof}
Every binary tree on $X$ has three internal nodes. The root has diameter $1$,
and each other internal node has diameter at least $1/2$. Thus every tree costs
at least $2$, attained by $(AC)(BD)$.

CL can first merge $A,B$, since their distance $1/2$ is tied for
the minimum. All remaining CL distances are then $1$, giving
total cost $1/2+1+1=5/2$. No CL output costs more, since its first
merge has diameter $1/2$ and each subsequent cluster has diameter at most $1$.
Hence the worst-case ratio on this instance, over tie-breaking choices, is
$(5/2)/2=5/4$.
\end{proof}

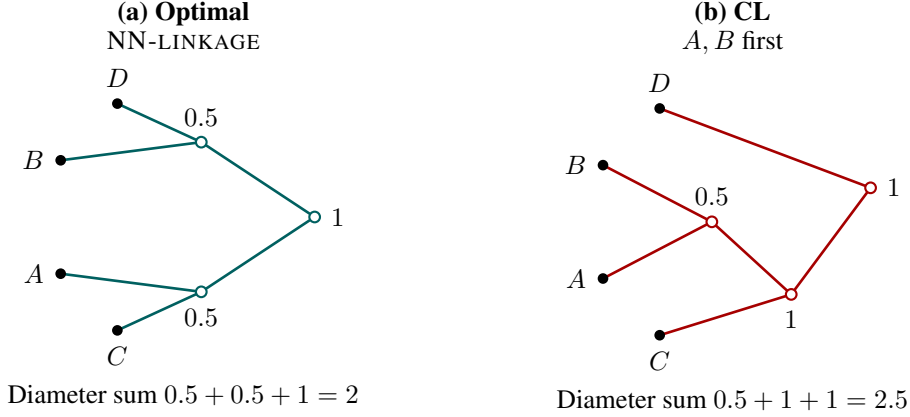
\begin{figure}[H]
\centering
\tikzset{
  example leaf/.style={circle,fill=black,inner sep=0pt,minimum size=4pt},
  example merge/.style={circle,fill=white,inner sep=0pt,minimum size=4pt,line width=0.8pt}
}
\begin{minipage}[t]{0.48\linewidth}
\centering
\textbf{(a) Optimal}\\
\nnlinkage{}
\par\smallskip
\begin{tikzpicture}[x=3cm,y=3cm,line width=1pt]
\coordinate (A) at (0,0.25);
\coordinate (B) at (0,0.75);
\coordinate (C) at (0.25,0);
\coordinate (D) at (0.25,1);
\coordinate (AC) at (0.62,0.17);
\coordinate (BD) at (0.62,0.83);
\coordinate (R) at (1.12,0.5);
\draw[teal!75!black] (A)--(AC)--(C) (B)--(BD)--(D) (AC)--(R)--(BD);
\node[example leaf,label=left:$A$] at (A) {};
\node[example leaf,label=left:$B$] at (B) {};
\node[example leaf,label=below:$C$] at (C) {};
\node[example leaf,label=above:$D$] at (D) {};
\node[example merge,draw=teal!75!black,label=below:$0.5$] at (AC) {};
\node[example merge,draw=teal!75!black,label=above:$0.5$] at (BD) {};
\node[example merge,draw=teal!75!black,label=right:$1$] at (R) {};
\end{tikzpicture}
\par
Diameter sum $0.5+0.5+1=2$
\end{minipage}
\hfill
\begin{minipage}[t]{0.48\linewidth}
\centering
\textbf{(b) CL}\\
$A,B$ first
\par\smallskip
\begin{tikzpicture}[x=3cm,y=3cm,line width=1pt]
\coordinate (A) at (0,0.25);
\coordinate (B) at (0,0.75);
\coordinate (C) at (0.25,0);
\coordinate (D) at (0.25,1);
\coordinate (AB) at (0.48,0.5);
\coordinate (ABC) at (0.83,0.18);
\coordinate (R) at (1.18,0.65);
\draw[red!65!black] (A)--(AB)--(B) (AB)--(ABC)--(C) (ABC)--(R)--(D);
\node[example leaf,label=left:$A$] at (A) {};
\node[example leaf,label=left:$B$] at (B) {};
\node[example leaf,label=below:$C$] at (C) {};
\node[example leaf,label=above:$D$] at (D) {};
\node[example merge,draw=red!65!black,label=above:$0.5$] at (AB) {};
\node[example merge,draw=red!65!black,label=below:$1$] at (ABC) {};
\node[example merge,draw=red!65!black,label=right:$1$] at (R) {};
\end{tikzpicture}
\par
Diameter sum $0.5+1+1=2.5$
\end{minipage}
\caption{The optimal tree $(AC)(BD)$, also returned by \nnlinkage{}, and a possible
CL tree $((AB)C)D$ on the original point set. Filled nodes retain
the point locations. Hollow nodes show merges and are labeled by cluster
diameter. Branches are schematic, not routing wires.}
\label{fig:complete-linkage-counterexample}
\end{figure}

\FloatBarrier
\subsection{\nnlinkage{} on This Point Set}
Testing \nnlinkage{}, trained on the $n=20$ uniform instances as before, we see
that it outputs the optimal tree $(AC)(BD)$.
\cref{fig:complete-linkage-counterexample} compares this tree with a
possible CL output.
\FloatBarrier
\subsection{Robustness Experiments}

We can force CL to make a non-optimal choice by decreasing point
$B$'s $y$-coordinate by $0.001$, making it closer to $A$. With $B=(0,0.749)$
and $A,C,D$ unchanged, $d(A,B)=0.499$ is the unique minimum pairwise distance,
so CL must merge $A,B$ first. It then returns $((AB)C)D$ with
cost $2.498$. The optimal tree is still $(AC)(BD)$, now with cost $2.001$,
and \nnlinkage{} still returns this tree despite the perturbation.
The costs of all 15 trees are given in
\cref{app:counterexample-computations}.

We can test how robust our trained model is compared to CL by
gradually moving points $A$ and $B$ closer together, instead of perturbing just
one point. We keep $C=(0.25,0)$ and $D=(0.25,1)$ fixed, and move $A$ and $B$
toward each other along the $y$-axis while keeping them centered around $y=0.5$.
Writing their distance as $s$, we have $A=(0,0.5-s/2)$ and $B=(0,0.5+s/2)$.
At the start, $s=0.5$, which gives the original instance.

At each tested value of $s$, we computed the optimal tree by checking all 15
trees, ran CL and \nnlinkage{}, and recorded their first merges
and costs. The ranges in the table summarize tested values only.
At $s=0.500$, the CL result depends on tie-breaking.

\begin{center}
\begin{tabular}{@{}ccccc@{}}
\shortstack{Tested\\$s$} & \shortstack{Optimal\\first merge} & \shortstack{CL\\first merge} & \shortstack{\nnlinkage{}\\first merge} & \shortstack{\nnlinkage{} cost /\\optimal cost} \\
\hline
$>0.500$ & $A,C$ or $B,D$ & $A,C$ or $B,D$ & $A,C$ or $B,D$ & $1.000$ \\
$0.500$ & $A,C$ or $B,D$ & $A,B$ & $A,C$ or $B,D$ & $1.000$ \\
$[0.445,0.500)$ & $A,C$ or $B,D$ & $A,B$ & $A,C$ or $B,D$ & $1.000$ \\
$0.440$ & $A,C$ or $B,D$ & $A,B$ & $A,B$ & $1.170$ \\
$0.305$ & $A,C$ or $B,D$ & $A,B$ & $A,B$ & $1.006$ \\
$0.300$ & $A,B$, $A,C$, or $B,D$ & $A,B$ & $A,B$ & $1.000$ \\
$<0.300$ & $A,B$ & $A,B$ & $A,B$ & $1.000$ \\
\end{tabular}
\end{center}

At $s=0.300$, $(AC)(BD)$, $((AB)C)D$, and $((AB)D)C$ are all optimal.
\nnlinkage{} preserves the optimal tree over a larger range of tested
perturbations than CL. However, it already returns a suboptimal
tree at $s=0.440$, before the optimal tree changes at $s=0.300$.

\FloatBarrier
\subsection{Supplementary Optimal Computations}
\label[appendix]{app:counterexample-computations}

\paragraph{Single-Point Perturbation.}
Here $B=(0,0.749)$, with $A,C,D$ unchanged from the original instance.

For the root cost, we can compute the distance from these pairs:
\[
\begin{aligned}
d(A,B)&=0.4990,\quad d(A,C)=0.5000,\quad d(B,D)=0.5010,\\
d(A,D)&=d(C,D)=1.0000,\quad d(B,C)=0.9990.
\end{aligned}
\]
These distances show that the diameter sum cost of the root is still $1.0$.

For the balanced trees,
\begin{center}
\begin{tabular}{lccc|c}
Tree & Node 1 & Node 2 & Root & Total \\
\hline
$(AC)(BD)$ & $d(A,C)=0.5000$ & $d(B,D)=0.5010$ & $1.0000$ & $\mathbf{2.0010}$ \\
$(AB)(CD)$ & $d(A,B)=0.4990$ & $d(C,D)=1.0000$ & $1.0000$ & $\mathbf{2.4990}$ \\
$(AD)(BC)$ & $d(A,D)=1.0000$ & $d(B,C)=0.9990$ & $1.0000$ & $\mathbf{2.9990}$ \\
\end{tabular}
\end{center}
For the unbalanced trees,
\begin{center}
\begin{tabular}{lccc|c}
Tree & Inner pair & Middle diam. & Root & Total \\
\hline
$((AB)C)D$ & $d(A,B)=0.4990$ & $\mathrm{diam}(\{A,B,C\})=0.9990$ & $1.0000$ & $\mathbf{2.4980}$ \\
$((AB)D)C$ & $d(A,B)=0.4990$ & $\mathrm{diam}(\{A,B,D\})=1.0000$ & $1.0000$ & $\mathbf{2.4990}$ \\
$((AC)B)D$ & $d(A,C)=0.5000$ & $\mathrm{diam}(\{A,C,B\})=0.9990$ & $1.0000$ & $\mathbf{2.4990}$ \\
$((AC)D)B$ & $d(A,C)=0.5000$ & $\mathrm{diam}(\{A,C,D\})=1.0000$ & $1.0000$ & $\mathbf{2.5000}$ \\
$((AD)B)C$ & $d(A,D)=1.0000$ & $\mathrm{diam}(\{A,D,B\})=1.0000$ & $1.0000$ & $\mathbf{3.0000}$ \\
$((AD)C)B$ & $d(A,D)=1.0000$ & $\mathrm{diam}(\{A,D,C\})=1.0000$ & $1.0000$ & $\mathbf{3.0000}$ \\
$((BC)A)D$ & $d(B,C)=0.9990$ & $\mathrm{diam}(\{B,C,A\})=0.9990$ & $1.0000$ & $\mathbf{2.9980}$ \\
$((BC)D)A$ & $d(B,C)=0.9990$ & $\mathrm{diam}(\{B,C,D\})=1.0000$ & $1.0000$ & $\mathbf{2.9990}$ \\
$((BD)A)C$ & $d(B,D)=0.5010$ & $\mathrm{diam}(\{B,D,A\})=1.0000$ & $1.0000$ & $\mathbf{2.5010}$ \\
$((BD)C)A$ & $d(B,D)=0.5010$ & $\mathrm{diam}(\{B,D,C\})=1.0000$ & $1.0000$ & $\mathbf{2.5010}$ \\
$((CD)A)B$ & $d(C,D)=1.0000$ & $\mathrm{diam}(\{C,D,A\})=1.0000$ & $1.0000$ & $\mathbf{3.0000}$ \\
$((CD)B)A$ & $d(C,D)=1.0000$ & $\mathrm{diam}(\{C,D,B\})=1.0000$ & $1.0000$ & $\mathbf{3.0000}$ \\
\end{tabular}
\end{center}
The minimum cost is $2.0010$, achieved by $(AC)(BD)$.

\paragraph{Symmetric Perturbations.}
 
$A=(0,0.28)$, $B=(0,0.72)$, $C=(0.25,0)$, $D=(0.25,1)$.
\[
\begin{aligned}
d(A,B)&=0.4400,\quad d(A,C)=d(B,D)=0.5300,\\
d(A,D)&=d(B,C)=0.9700,\quad d(C,D)=\mathrm{diam}(X)=1.0000.
\end{aligned}
\]
 
\begin{center}
\begin{tabular}{lccc|c}
Tree & Component 1 & Component 2 & Root & Total \\
\hline
$(AC)(BD)$ & $0.5300$ & $0.5300$ & $1.0000$ & $\mathbf{2.0600}$ \\
$(AB)(CD)$ & $0.4400$ & $1.0000$ & $1.0000$ & $\mathbf{2.4400}$ \\
$(AD)(BC)$ & $0.9700$ & $0.9700$ & $1.0000$ & $\mathbf{2.9400}$ \\
$((AB)C)D$ & $0.4400$ & $0.9700$ & $1.0000$ & $\mathbf{2.4100}$ \\
$((AB)D)C$ & $0.4400$ & $0.9700$ & $1.0000$ & $\mathbf{2.4100}$ \\
$((AC)B)D$ & $0.5300$ & $0.9700$ & $1.0000$ & $\mathbf{2.5000}$ \\
$((AC)D)B$ & $0.5300$ & $1.0000$ & $1.0000$ & $\mathbf{2.5300}$ \\
$((AD)B)C$ & $0.9700$ & $0.9700$ & $1.0000$ & $\mathbf{2.9400}$ \\
$((AD)C)B$ & $0.9700$ & $1.0000$ & $1.0000$ & $\mathbf{2.9700}$ \\
$((BC)A)D$ & $0.9700$ & $0.9700$ & $1.0000$ & $\mathbf{2.9400}$ \\
$((BC)D)A$ & $0.9700$ & $1.0000$ & $1.0000$ & $\mathbf{2.9700}$ \\
$((BD)A)C$ & $0.5300$ & $0.9700$ & $1.0000$ & $\mathbf{2.5000}$ \\
$((BD)C)A$ & $0.5300$ & $1.0000$ & $1.0000$ & $\mathbf{2.5300}$ \\
$((CD)A)B$ & $1.0000$ & $1.0000$ & $1.0000$ & $\mathbf{3.0000}$ \\
$((CD)B)A$ & $1.0000$ & $1.0000$ & $1.0000$ & $\mathbf{3.0000}$ \\
\end{tabular}
\end{center}
 
The minimum is $2.0600$, achieved uniquely by $(AC)(BD)$. The optimal tree is unchanged from the original instance, but the model's output at this $s$ matches $((AB)C)D$ instead, costing $2.4100$, matching CL's suboptimal trajectory.

The optimal tree shifts when the distance between points $A$ and $B$ becomes $0.300$.

$A=(0,0.35)$, $B=(0,0.65)$, $C=(0.25,0)$, $D=(0.25,1)$.
\[
\begin{aligned}
d(A,B)&=0.3000,\quad d(A,C)=d(B,D)=0.6000,\\
d(A,D)&=d(B,C)=0.9000,\quad d(C,D)=\mathrm{diam}(X)=1.0000.
\end{aligned}
\]
 
\begin{center}
\begin{tabular}{lccc|c}
Tree & Component 1 & Component 2 & Root & Total \\
\hline
$(AC)(BD)$ & $0.6000$ & $0.6000$ & $1.0000$ & $\mathbf{2.2000}$ \\
$(AB)(CD)$ & $0.3000$ & $1.0000$ & $1.0000$ & $\mathbf{2.3000}$ \\
$(AD)(BC)$ & $0.9000$ & $0.9000$ & $1.0000$ & $\mathbf{2.8000}$ \\
$((AB)C)D$ & $0.3000$ & $0.9000$ & $1.0000$ & $\mathbf{2.2000}$ \\
$((AB)D)C$ & $0.3000$ & $0.9000$ & $1.0000$ & $\mathbf{2.2000}$ \\
$((AC)B)D$ & $0.6000$ & $0.9000$ & $1.0000$ & $\mathbf{2.5000}$ \\
$((AC)D)B$ & $0.6000$ & $1.0000$ & $1.0000$ & $\mathbf{2.6000}$ \\
$((AD)B)C$ & $0.9000$ & $0.9000$ & $1.0000$ & $\mathbf{2.8000}$ \\
$((AD)C)B$ & $0.9000$ & $1.0000$ & $1.0000$ & $\mathbf{2.9000}$ \\
$((BC)A)D$ & $0.9000$ & $0.9000$ & $1.0000$ & $\mathbf{2.8000}$ \\
$((BC)D)A$ & $0.9000$ & $1.0000$ & $1.0000$ & $\mathbf{2.9000}$ \\
$((BD)A)C$ & $0.6000$ & $0.9000$ & $1.0000$ & $\mathbf{2.5000}$ \\
$((BD)C)A$ & $0.6000$ & $1.0000$ & $1.0000$ & $\mathbf{2.6000}$ \\
$((CD)A)B$ & $1.0000$ & $1.0000$ & $1.0000$ & $\mathbf{3.0000}$ \\
$((CD)B)A$ & $1.0000$ & $1.0000$ & $1.0000$ & $\mathbf{3.0000}$ \\
\end{tabular}
\end{center}
 
Here $(AC)(BD)$, $((AB)C)D$, and $((AB)D)C$ all tie at $2.2000$. This is the exact point where the true optimal tree stops being uniquely $(AC)(BD)$, and the model's and CL's costs equal that of the optimal tree.

\end{document}